\documentclass[ejs]{imsart}

\usepackage{amsthm,amsmath}
\usepackage[hyperfootnotes]{hyperref}
\usepackage{amssymb}

\usepackage{graphicx}

\usepackage{subcaption}

\usepackage{amsfonts}
\RequirePackage[numbers]{natbib}
\RequirePackage[colorlinks,citecolor=blue,urlcolor=blue]{hyperref}

\doi{10.1214/154957804100000000}
\pubyear{0000}
\volume{0}
\firstpage{1}
\lastpage{1}

\startlocaldefs
\newtheorem{remark}{Remark}[section]

\newtheorem{theorem}{Theorem}[section]
\newtheorem*{theorem*}{Theorem}

\newtheorem{definition}{Definition}[section]

\newtheorem{corollary}{Corollary}[theorem]
\newtheorem{lemma}[theorem]{Lemma}
\newtheorem{proposition}[theorem]{Proposition}

\newcommand{\GMM}{\mbox{GMM}(\mu^*,\pi)}
\newcommand{\Ei}{\mathbb{E}_i}
\newcommand{\U}{\mathcal{U}_{\lambda}}
\newcommand{\E}{\mathbb{E}}
\endlocaldefs

\begin{document}

\begin{frontmatter}

\title{Improved Convergence Guarantees for Learning Gaussian Mixture Models	by EM and Gradient EM \thanks{	This research was partially supported by the Israeli Council for Higher Education (CHE) via the Weizmann Data Science Research Center. This research was also partially supported by a research grant from the Estate of Tully and Michele Plesser.}}

\runtitle{Improved Convergence Guarantees for EM}
\begin{aug}
	\author{\fnms{Nimrod} \snm{Segol}\ead[label=e1]{nimrod.segol@weizmann.ac.il}}
	\and
	\author{\fnms{Boaz} \snm{Nadler}\ead[label=e2]{boaz.nadler@weizmann.ac.il}}
	
	\address{Department of Computer Science and Applied Mathematics\\
		Weizmann Institute of Science\\
		Rehovot, Israel\\
		\printead{e1,e2}}
\runauthor{Segol and Nadler}
\end{aug}

\begin{abstract}
	 We consider the problem of estimating the parameters a Gaussian Mixture Model with $K$ components of known weights, all with an identity covariance matrix. We make two contributions.
	First, at the population level, we present a sharper analysis of the local convergence of EM and gradient EM, compared to previous works. Assuming a separation of $\Omega(\sqrt{\log K})$, we prove convergence of both methods 
	to the global optima from an initialization region larger than those of previous works. 
	Specifically, the initial guess of each component can be as far as (almost) half its distance to the nearest Gaussian. This is essentially the largest possible  contraction region. 
	Our second contribution are improved sample size requirements for accurate estimation by EM and gradient EM. In previous works, the required number of samples had a quadratic dependence on the maximal separation between the $K$ components, and the resulting error estimate increased linearly with this maximal separation. In this manuscript we show that both quantities depend only logarithmically on the maximal separation. 
\end{abstract}





\end{frontmatter}


\section{INTRODUCTION}

Gaussian mixture models (GMMs) are a widely used statistical model going back to Pearson  \cite{pearson1894contributions}. 
In a GMM each sample $x\in \mathbb{R}^{d}$ is drawn from one of   $K$ components according to mixing weights $\pi_1, \ldots ,\pi_K>0$ with $\sum_{i=1}^K\pi_i=1$. Each component follows a Gaussian distribution with  mean $\mu_i^*\in \mathbb{R}^d$ and covariance $\Sigma_i\in \mathbb{R}^{d\times d}$.
In this work, we focus on the important special case of $K$ spherical Gaussians with identity covariance matrix, with a corresponding density function
\begin{equation}
f_X(x)=\sum_{i=1}^K \frac{\pi_i}{(2\pi)^{\frac{d}{2}}} e^{-\frac{\|x-\mu_i^*\|^2}{2}}. 
\label{Eq:GMMdensity}
\end{equation} 
For simplicity, as in \cite{yan2017convergence,zhao2020statistical}, we assume the weights $\pi_i$ are known. 

Given $n$ i.i.d. samples from the distribution \eqref{Eq:GMMdensity}, a fundamental problem is to estimate the vectors $\mu_i^*$ of the $K$ components.
Beyond the number of components $K$ and the dimension $d$, 
the difficulty of this problem is characterized by the following key quantities: The smallest and largest separation between the 
cluster centers,  
\begin{equation}
R_{\min}= \min_{i\ne j}\|\mu_i^*- \mu_j^*\|,\quad R_{\max}= \max_{i\ne j}\|\mu_i^*- \mu_j^*\|,  \label{R_min_R_max_def}
\end{equation} 
the minimal and maximal weights and their ratio,  \begin{equation}
\pi_{\min}= \min_{i\in[K]}\pi_i, \quad \pi_{\max}= \max_{i\in [K]}\pi_i, \quad \theta = \frac{\pi_{\max}}{\pi_{\min}}. \label{Pi_min_pi_max}
\end{equation}

In principle, one could estimate $\mu_i^*$ 
by maximizing the likelihood of the observed data. However, as the log-likelihood is non-concave, this problem is computationally challenging.
A popular alternative approach is based on 
the EM algorithm \citep{dempster1977maximum}, and  variants thereof, such as gradient EM.
These iterative methods require an initial guess $(\mu_1,\ldots,\mu_K)$ of the $K$ cluster centers. 
Classical  results show that regardless of the initial guess, the values of the likelihood function after each EM iteration are non decreasing. Furthermore, under fairly general conditions, the EM algorithm converges to a stationary point or a local optima \citep{xu1996convergence, wu1983convergence}. The success of these methods to converge to an accurate solution depend critically on the accuracy of the initial guess \citep{jin2016local}.

In this work we study the ability of the popular EM and gradient EM algorithms to 
accurately estimate the parameters of the GMM in 
(\ref{Eq:GMMdensity}). 
%
Two quantities of particular interest are: (i) 
the size of the initialization region and the minimal separation that guarantee convergence to the global optima. Namely, how small can $R_{\min}$ be and how large can $\|\mu_i-\mu_i^*\|$, and still have convergence of EM to the global optima in the population setting; 
and (ii) the required sample size, 
and its dependence on the problem parameters, that guarantees EM to find accurate solutions, with high probability.  

We make the following contributions:
First, we present an improved analysis of the local convergence of EM and gradient EM, at the population level, assuming an infinite number of samples.
In Theorems \ref{main_EM_population_theorem} and \ref{grad_EM_population_theorem} we prove  their
convergence under the largest possible initialization region, while requiring a separation  $R_{\min}=\Omega\left(\sqrt{\log  K}\right)$.
For example, consider the case of equal weights $\pi_i=1/K$, and an initial guess
that satisfies $\|\mu_i-\mu_i^*\|\leq \lambda \min_{j\neq i}\|\mu_j^*-\mu_i^*\|$ for all $i$, 
with $\lambda < 1/2$. 
Then, a separation 
$R_{\min} \ge C(\lambda)  \sqrt{\log K}$, with an explicit $C(\lambda)$, suffices to ensure that the population EM and gradient EM algorithms converge to the true means at a linear rate. 

Let us compare our results to several recent works that derived convergence guarantees for EM and gradient EM. \cite{zhao2020statistical} and  \cite{yan2017convergence} proved local convergence to the global optima under a much larger minimal separation of $R_{\min}\ge C \sqrt{\min(d,K)\log K}$. In addition, the requirement on the initial estimates had a dependence on the maximal separation, $\|\mu_i-\mu_i^*\|\le \frac12R_{\min}-C_1\sqrt{\min(d,K)\log \max(R_{\max}, K^3)}$ for a universal constant $C_1$. These results were significantly
improved by \cite{kwon2020algorithm}, who  proved the local 
convergence of the EM algorithm for the more general case of spherical Gaussians with unknown weights and variances.
They required a far less restrictive minimal separation $R_{\min}\ge C \sqrt{\log K}$, with a constant $C\ge 64$, and their initialization was restricted to $\lambda<\frac1{16}$. We should note that no particular effort was made to optimize these
constants. 
In comparison to these works, we allow the largest
possible 
initialization region $\lambda<\frac{1}2$, with no dependence on $R_{\max}$. Also, for small values $\lambda\leq 1/16$, our resulting constant $C$ is roughly 6 times smaller that
that of \cite{kwon2020algorithm}.

Our second contribution concerns the required sample size to ensure accurate estimation by the EM and gradient EM algorithms. 
Recently, \cite{kwon2020algorithm} proved that with a number of samples $n=\tilde{\Omega}(d/\pi_{\min})$, a sample splitting variant of EM is statistically optimal. In this variant, the $n$ samples are split into $B$ distinct batches, with
each EM iteration using a separate batch. 
In contrast, for the standard EM and gradient EM algorithms,
weaker results have been established so far. Currently, the best known sample requirements for 
EM are $n = \tilde{\Omega}(K^3dR_{\max}^2/R_{\min}^2)$,
whereas for gradient EM, 
$ n = \tilde{\Omega}(K^6dR_{\max}^6/R_{\min}^2)$. In addition, the bounds for the resulting errors increase linearly with $R_{\max}$, see
\citep{yan2017convergence,zhao2020statistical}.
Note that in these two results, the 
required number of samples increases at least quadratically with the maximal separation
between clusters, even though increasing $R_{\max}$ should make the problem easier. 
In Theorems 
\ref{sample_em_main_theorem} and \ref{sample_grad_em_main_theorem}, 
we prove that for an initialization region with parameter $\lambda$  strictly smaller than half, 
the EM and gradient EM algorithms yield accurate estimates with sample size $\tilde{\Omega}(K^3d)$. In particular, both our sample size requirements 
and the bounds on the error of the EM and gradient EM have only a logarithmic dependence on $R_{\max}$. 



Our results on the  initialization region and minimal separation stem from a careful analysis of the weights in the  EM update and their effect on the estimated cluster centers. Similarly to \cite{kwon2020algorithm}, we upper bound the expectation of the $i$-th weight when the data is drawn from a different component $j\ne i$ and show that it is \textit{exponentially small} in the distance between the centers of the $i$ and $j$ components. 
We make use of the fact that all Gaussians have the same covariance to reduce the expectation to  one dimension  and directly upper bound the one dimensional integral. This allows us to derive a {sharper} bound  compared to \cite{kwon2020algorithm}  from which we obtain a larger contraction region for the population EM and gradient EM algorithms. 
Our analysis of the finite sample behavior of EM and
gradient EM follows the general strategy of 
\cite{zhao2020statistical}. Our improved results
rely on tighter bounds on the sub-Gaussian norm of the weights in the EM update which do not depend on the distance between the clusters.


\subsection{Previous work}
Over the past decades, several approaches to estimate the parameters of Gaussian mixture models were proposed. In addition, many works derived theoretical guarantees for these methods as well as information-theoretic lower bounds on the number of samples required for accurate estimation. 
Significant efforts were made in understanding
whether GMMs can be learned efficiently both
from a computational perspective, namely in polynomial run time, and from a statistical view, namely with a number of samples polynomial in the problem parameters.

Method of moments approaches  \citep{kalai2010efficiently,moitra2010settling, hardt2015tight} can accurately estimate 
the parameters of general GMMs with $R_{\min}$ arbitrarily small, at the cost of sample complexity, and thus also run time, that is exponential in the number of clusters. 
\cite{hsu2013learning} showed that a method of moments type algorithm can recover the parameters of spherical GMMs with arbitrarily close cluster centers in polynomial time, under the additional assumption that the components centers are affinely independent. This assumption implies that $d\geq K$. 

Methods based on dimensionality  reduction  \citep{dasgupta1999learning, achlioptas2005spectral,arora2005learning,kannan2008spectral,vempala2004spectral} can accurately estimate the parameters of a GMM in polynomial time in the dimension and number of clusters, under conditions on the separation of the clusters' centers. In particular, \cite{vempala2004spectral}
proved that accurate recovery is possible with a 
minimal separation of $R_{\min}= \Omega(\min(K,d)^{\frac{1}{4}})$.


In  general,  it is not possible to learn the parameters of a GMM with number of samples that is polynomial in the number of clusters,
see \citep{moitra2010settling} for an explicit example. \cite{regev2017learning} showed that for any function $\gamma(K)= o(\sqrt{\log K})$  one can find two spherical GMMs, both with $R_{\min}= \gamma(K)$ such that no algorithm with polynomial sample complexity can distinguish between them. \cite{regev2017learning} also presented  a variant of the EM algorithm
that provably learns the parameters of a GMM with separation $\Omega( \sqrt{\log K})$, with polynomial sample complexity, but run time exponential in the number of components.


More closely related to our manuscript, are several works that studied the ability of EM and variants thereof to accurately estimate the parameters of a GMM. 
\cite{JMLR:v8:dasgupta07a} showed that 
with a separation of $\Omega(d^{\frac{1}{4}})$,
a two-round variant of EM produces accurate estimates
of the cluster centers. 
A significant advance was made by \cite{balakrishnan2017statistical}, who developed new techniques to analyze the local convergence of EM for rather general latent variable models. 
In particular, for a GMM with $K=2$ components of equal weights, 
they proved that the EM algorithm converges locally at a linear rate provided that the distance between the components is at least some universal constant. These results were extended  in \cite{xu2016global} and \cite{daskalakis2017ten} where a full description of the initialization region for which the population EM algorithm learns a mixture of any two equally weighted Gaussians was given. 
As already mentioned above, the three works that are directly related to our work, and to which we compare in detail in Section \ref{sec:Main_Results} are \cite{yan2017convergence}, \cite{zhao2020statistical} and 
\cite{kwon2020algorithm}.

\section{PROBLEM SETUP AND NOTATIONS}

\subsection{Notations} \label{notatins}

We write $X\sim \GMM$ for a random variable with density given 
by Eq. (\ref{Eq:GMMdensity}). 
The distance between cluster means is denoted by $R_{ij}= \|\mu_i^*-\mu_j^*\|$. We set $R_i= \min_{j\ne i}R_{ij}$.
Expectation of a function $f(X)$ with respect to $X$ is denoted by $\mathbb{E}_X[f(X)]$, or when clear from context simply by  $\mathbb{E}[f(X)]$. For simplicity of notation, 
we shall write $\Ei[f(X)]=\mathbb{E}_{X\sim \mathcal N(\mu_i^*,I_d)}[f(X)]$.    
For a vector $v$ we denote by $\|v\|$ its Euclidean norm. 
For a matrix $A$, we denote its operator norm by $\|A\|_{op}=\max_{\|x\|=1}\|Ax\|$. 
Finally, we denote by 
$\mu= (\mu_1^\top, \ldots, \mu_K^\top)^\top\in \mathbb{R}^{Kd}$ the concatenation of $\mu_1,\ldots, \mu_K\in\mathbb{R}^d$. 

As in previous works, we consider the following error measure for the quality of an estimate $\mu$ of the true means,  
\begin{equation}
E(\mu)= \max_{i\in [K]}\|\mu_i-\mu_i^*\|.    
\nonumber
\end{equation}
We will see that in the population case we can restrict our analysis to the space spanned by the $K$ true cluster means and the $K$ cluster estimates. It will therefore be convenient to define $
d_0= \min(d,2K)$.
For any $0<\lambda<\frac12$ we define the region  \begin{equation}
\mathcal{U}_{\lambda}=\left\{\mu\in \mathbb{R}^{Kd}:\|\mu_i-\mu_i^*\| \le \lambda R_i 
\quad \forall i\in[K]\right\}.
\label{U}
\end{equation}
For future use we define the following function
which will play a key role in our analysis, 
\begin{equation}
c(\lambda)=\frac18\left(\frac{1-2\lambda}{1+2\lambda}\right)^2.
\label{c(lambada)}
\end{equation}


\subsection{Population and Sample EM}


Given an estimate $(\mu_1, \ldots, \mu_K)$ of the $K$ centers, for any $x\in \mathbb{R}^d$ and $i\in [K]$ let
\begin{equation}
w_i(x,\mu) 
=  \frac{\pi_ie^{-\frac{\|x-\mu_i\|^2}{2}}} {\sum_{j=1}^K \pi_je^{-\frac{\|x-\mu_j\|^2}{2}}}. \label{Eq:w}
\end{equation} 
The population EM update, denoted by $\mu^+=(\mu^+_1, \ldots, \mu_K^+)$ is given by 
\begin{equation}
\mu^+_i=\frac{\mathbb{E}_X[w_i(X,\mu)X]}{\mathbb{E}_X[w_i(X,\mu)]}, \quad \forall i\in [K].
\label{populationEmupdate}   
\end{equation}  
The population gradient EM update with a step size $s>0$ is defined by  
\begin{equation}
\mu_i^+ 
=  \mu_i+s\mathbb{E}_X\left[w_i(X,\mu)(X-\mu_i)\right] , \quad \forall i\in [K]. 
\label{gradient_EM_update}
\end{equation}

Given an observed set of $n$ samples $X_1, \ldots, X_n\sim X$, the sample EM and sample gradient EM updates follow by replacing the expectations in \eqref{populationEmupdate} and \eqref{gradient_EM_update} 
with their empirical counterparts. For the EM, the update is 
\begin{equation}
\mu^+_i = \frac{\sum_{\ell=1}^nw_i(X_\ell,\mu)X_{\ell}}{\sum_{\ell=1}^nw_i(X_\ell,\mu)}, \quad \forall i\in [K]
\label{eq:sample_em_update}
\end{equation}
and  for the gradient EM \begin{equation}
\mu_i^+=  \mu_i+s \frac1n \sum_{\ell=1}^n w_i(X_{\ell},\mu)(X_{\ell}-\mu_i), \quad \forall i\in [K]. 
\label{eq:sample_grad_em_update}
\end{equation}

In this work, we study the convergence of EM and gradient EM, both in the population setting and with a finite number of samples. In particular we are interested in sufficient conditions on the initialization and on the separation of the GMM components that ensure convergence to accurate solutions. 

\section{LOCAL CONVERGENCE OF EM AND GRADIENT EM}
\label{sec:Main_Results}
\subsection{Population EM} 

As in previous works, we first study the convergence of EM in the population case and then build upon this analysis to study the finite sample setting.
Informally, our main result in this section is that for any fixed $\lambda\in (0,\frac12)$ and an initial estimate $\mu\in \U$, there exists a constant $C(\lambda)$ such that for any mixture with $R_{\min} \gtrsim C(\lambda)\sqrt{\log\frac{1}{\pi_{\min}} }$ the estimation error of a single population EM update \eqref{populationEmupdate} decreases by a multiplicative factor strictly less than $1$. This, in turn, implies convergence of the population EM to the global optimal solution $\mu^*$.
Formally, our result is stated in the following theorem. 

\begin{theorem} \label{main_EM_population_theorem}	Set $\lambda\in (0,\frac12)$. Let $X\sim\GMM$ with 
	\begin{equation} 
	\label{minimal_separation}
	R_{\min} \ge \sqrt{\frac{4}{c\left(\lambda\right)}\log\frac{32\left(K-1\right)\sqrt{14\left(1+\theta\right)}}{3\pi_{\min}c(\lambda)}}
	\end{equation}
	where  $c(\lambda)$ and $\theta$ are as defined in \eqref{c(lambada)} and \eqref{Pi_min_pi_max}, respectively. Then for any $\mu\in \U$ it holds that $E(\mu^t)\le \frac1{2^t} E(\mu)$ where $\mu^t$ is the $t$-th iterate of the  population EM update \eqref{populationEmupdate}  initialized at $\mu$. 
\end{theorem}

We derive a similar result 
for  gradient EM.  
\begin{theorem} \label{grad_EM_population_theorem}
	Set $\lambda\in (0,\frac12)$. Let $X\sim\GMM$ with $R_{\min}$ satisfying \eqref{minimal_separation}. Then for any $s\in\left(0,\frac1{\pi_{\min}}\right)$ and any $\mu\in \U$  it holds that $E(\mu^t)\le \gamma^{t} E(\mu)$ where $\mu^t$ is the $t$-th iterate of the population gradient EM update \eqref{gradient_EM_update} with step size $s$ and $\gamma= 1-\frac38s\pi_{\min}$.
\end{theorem}

The proof of Theorem \ref{main_EM_population_theorem} appears in Section \ref{section:population}  with the technical details deferred to the appendix. The proof of Theorem \ref{grad_EM_population_theorem} is similar and 
appears in full in the appendix. 


It is interesting to compare 
Theorems \ref{main_EM_population_theorem} and 
\ref{grad_EM_population_theorem} 
to several recent works, in terms of both the size of
the initialization region, and the requirements on the minimal separation.  
\cite{yan2017convergence} and \cite{zhao2020statistical}  
assumed a separation $R_{\min}=\Omega(\sqrt{d_0\log K})$ and 
proved local convergence of the gradient EM and of the EM algorithm, for an initialization region of the following form, with $C_1$ a universal constant,  
\begin{equation}
\max_{i\in [K]} \|\mu_i-\mu_i^*\| \le  \frac12R_{\min}-C_1\sqrt{d_0\log \max(R_{\max}, K^3)}.
\nonumber
\end{equation} 
 Recently, \cite{kwon2020algorithm} 
significantly improved these works, proving convergence of population EM with a much smaller separation
$R_{\min} \geq 64 \sqrt{\log(\theta K)}$. Moreover, 
they considered the more general and challenging case where the Gaussians may have different variances and the EM algorithm estimates not only the Gaussian centers, but also their weights and variances.
However, they proved convergence only for 
an initialization region $\U$ with 
$\lambda \leq \frac{1}{16}$.

Our results improve upon these works in several aspects. 
First, in comparison to the contraction region of \cite{yan2017convergence},
our theorem allows the largest possible initialization
region $\|\mu_i-\mu_i^*\|<\frac12 R_i$, 
with no dependence on the other
problem parameters 
$d_0,K$ and $R_{\max}$. This initialization region is optimal as there exists GMMs and  initializations $\mu$ with $\|\mu_i-\mu_i^*\|=\frac{1}{2}R_{i}$  such that the EM algorithm, even at the population level,  will not converge to values that are close to the true parameters. 

Second, in comparison to the result of \cite{kwon2020algorithm}, we allow $\lambda$ to be as large as $\frac12$. Also, for $\lambda<\frac1{16}$, our requirement on $R_{\min}$ 
is nearly one order of magnitude smaller. 
For example, for a balanced mixture with $\pi_{\min}=\frac1K$,  the right hand side of \eqref{minimal_separation} reads 
\begin{equation*}
\sqrt{\frac{4}{c\left(\lambda\right)}\left(\log\left(K^{2}\right) +\log\frac{32\sqrt{28}}{3c(\lambda)}\right)}.
\end{equation*} 
An initialization region $\|\mu_i-\mu_i^*\|\le \frac1{16}R_i$ leads to a separation requirement $R_{\min } \ge 10.3\sqrt{\log( K)+ 6.6}$, which is much smaller
than $64\sqrt{\log K}$.

We remark on the necessity of our  assumptions  on the separation and initialization in Theorems \ref{main_EM_population_theorem} and \ref{grad_EM_population_theorem}. In general, given only a polynomial number of samples, 
a separation of $R_{\min}=\Omega(\sqrt{\log K})$ is necessary to accurately estimate the parameters of a
GMM regardless of the estimation method \cite{regev2017learning}. 
With infinitely many samples and sufficiently close initial estimates, the EM algorithm may still converge to the global optimum even with $\Omega(1)$ separation. 
However, to the best of our knowledge, a precise characterization of the attraction region to the true parameters is still an open problem.   Next, the separation requirement \eqref{minimal_separation} in our theorems  depends inversely on $c(\lambda)$. Therefore, as $\lambda\to \frac12$ the requirement on $R_{\min}$ becomes more restrictive.
Simulation  results, see Figure 1b in Section \ref{sec:simulation}, suggest that this dependence of the separation requirement on the initialization may be significantly relaxed. We conjecture that the EM and gradient EM algorithms converge when $R_{\min} \ge C\sqrt{\log\frac{K}{\pi_{\min}}}$ for a universal constant $C$ and $\|\mu_i-\mu_i^*\|<\frac12R_i$.

\subsection{Sample EM}
We now present our results on the EM and gradient EM algorithms for the finite sample case.

\begin{theorem} \label{sample_em_main_theorem}
	Set $\lambda\in (0,\frac12), \delta\in (0,1)$. Let $X_1,\ldots, X_n \overset{i.i.d.}{\sim} \GMM$ with $R_{\min}$ satisfying  \eqref{minimal_separation}.
	Suppose that $n$ is sufficiently large so that \begin{equation}
	\frac{n}{\log n} >
	C \frac{Kd\log\left(\frac{\tilde{C}}{\delta}\right)}{\pi_{\min}}\max\left(1,\frac{1}{(1-2\lambda)^2\lambda^2\pi_{\min}R_{\min}^2} \right). 
	\label{minimal_sample}
	\end{equation} 
	where $C$ is a universal constant and $\tilde{C} = 100K^2 R_{\max}(\sqrt d + 2R_{\max})^2$. 
	Assume an initial estimate $\mu \in \mathcal{U}_{\lambda}$ and let $\mu^t$ be the $t$-th iterate of the sample EM update \eqref{eq:sample_em_update}. Then with probability at least $1-\delta$, 
	for all iterations $t$, 
	$\mu^t\in\U$ and
	\begin{equation}
	\|\mu_i^t-\mu_i^*\| \le \frac{1}{2^t}E(\mu)+ 	\frac{C_1}{(1-2\lambda)\pi_{i}}\sqrt{\frac{Kd\log\frac{\tilde{C}n}{\delta}}{n}} \label{convergence_up_to_statistical_errormain}
	\end{equation}
	for a suitable absolute constant $C_1$.
\end{theorem}   

\begin{theorem} \label{sample_grad_em_main_theorem}
	Set $\lambda\in (0,\frac12), \delta\in (0,1)$ . Let $X_1,\ldots, X_n \overset{i.i.d.}{\sim} \GMM$ with $R_{\min}$ satisfying \eqref{minimal_separation}. Set $s\in \left(0,\frac1{\pi_{\min}}\right)$  and suppose that $n$ is sufficiently large so that  
	\begin{equation}
	\frac{n}{\log n} >\frac{CKd\log\frac{\tilde{C}}{\delta}}{\pi_{\min}^{2}}\max_{i\in [K]}\frac{\max\left(\lambda^{2}R_{i}^{2},\frac1{\left(1-2\lambda\right)^{2}}\right)}{\lambda^{2}R_{i}^{2}}
	\label{minimal_sample_grad}
	\end{equation} 
	where $C$ is a universal constant and $\tilde{C} = 36K^2R_{\max}(\sqrt d + 2R_{\max})^2$. 
	Assume an initial estimate $\mu \in \mathcal{U}_{\lambda}$ and let $\mu^t$ be the $t$-th iterate of the sample gradient  EM update \eqref{eq:sample_grad_em_update} with step size $s$. Then with probability at least $1-\delta$, $\mu^t\in\U$ for all $t$, and \begin{equation}
	\|\mu_i^t-\mu_i^*\| \le \gamma^t E(\mu)+\frac{C_1}{\pi_{i}}\max\left(\frac1{1-2\lambda}, \lambda R_i\right) \sqrt{\frac{Kd\log\left(\frac{\tilde{C}n}{\delta}\right)}{n}} \label{convergence_up_to_statistical_errorgrad}
	\end{equation}
	where $\gamma=1-\frac38s\pi_{\min}$  and $C_1$ is a suitable absolute constant.
\end{theorem}

The main idea in the proofs of Theorems \ref{sample_em_main_theorem} and \ref{sample_grad_em_main_theorem} is to show the uniform convergence, inside the initialization region $\U$, of the sample update to the population update. The sample size requirements \eqref{minimal_sample} and \eqref{minimal_sample_grad} are such that the resulting error of a single update of the EM and gradient EM algorithms is sufficiently small to ensure that the updated means are in the contraction region $\U$. This, combined with the convergence of the population update, yields the required result. We outline the main steps of the proof in Section \ref{section:sample} with more technical details deferred to the appendix.


Let us compare Theorems \ref{sample_em_main_theorem}
and \ref{sample_grad_em_main_theorem} to previous results, in terms of required sample size and bounds on the estimation error.
The strongest result to date, due to \cite{kwon2020algorithm}, considered a variant of the EM algorithm, whereby the samples are split into $B$ separate batches, and at each iteration $t$ (with $1\leq t\leq B$), the sample EM algorithm is run only using the data of the $t$-th batch. They showed that to achieve an error $E(\mu^B)\leq \epsilon$, the required sample size is 
$\tilde{\Omega}(\frac{d}{\pi_{\min} \epsilon^2} )$. 
The best known bounds without sample splitting were derived
by \cite{yan2017convergence} and by \cite{zhao2020statistical}.
The error guarantee for gradient EM is $\tilde O(n^{-1/2} \max (K^3 R_{\max}^3\sqrt{d},R_{\max}d))$,
whereas for EM it is
$\tilde O(n^{-1/2} R_{\max}\sqrt{Kd}/\pi_{\min})$.
The sample size requirements for gradient EM are  
$\frac{n}{\log n} = \tilde\Omega(\max(K^3 R_{\max}^3\sqrt{d},R_{\max}d)^2 / R_{\min}^2)$ and $\frac{n}{\log n} = \tilde\Omega(\frac{Kd}{\pi_{\min}^2}\max(1,R_{\max}^2/R_{\min}^2))$ for EM. 
Note that these bounds have a dependence on 
the {\em maximal} separation $R_{\max}$. In particular, even though intuitively, as $R_{\max}$ increases the problem should become easier, these error bounds increase linearly with $R_{\max}$ and the required sample size increases quadratically with $R_{\max}$. In contrast, in our two theorems above there is a dependence on $1/(1-2\lambda)$, which is strictly smaller than $R_{\max}$ by the separation condition \eqref{minimal_separation}. Thus, 
for $\lambda$ bounded away from $1/2$, 
there is only a  
logarithmic dependence on $R_{\max}$.  We believe that
with further effort, the dependence on $R_{\max}$ can be fully eliminated. 

We note that both the minimal sample size requirement in Equations \eqref{minimal_sample} and \eqref{minimal_sample_grad}   and the bounds on the error in Equations \eqref{convergence_up_to_statistical_errormain} and \eqref{convergence_up_to_statistical_errorgrad}  could probably be improved. Indeed, \cite{kwon2020algorithm} proved that the sample splitting variant of the EM algorithm yields accurate estimates with only $ \tilde{\Omega}(Kd)$ samples. Numerical results, see Section \ref{sec:simulation}, suggest that the error of  the classical EM algorithm  depends only on $\sqrt{Kd}$.

\section{PROOF  FOR THE POPULATION EM}\label{section:population}

Our strategy is similar to \cite{yan2017convergence} and \cite{zhao2020statistical}: We bound the error of a single update, $\|\mu_i^+-\mu_i^*\|$ in terms of 
$\E_X[w_j(X,\mu)]$ and 
$\E_X[\nabla_{\mu} w_j(X,\mu)(X-\mu_j)]$, which in turn depend on
their expectations
with respect to individual Gaussian components. 
Our key result on the latter expectation is the following Proposition, whose proof appears in the appendix. 

\begin{proposition}\label{Lem:w_exponentialy_small}
	Set $0<\lambda < \frac{1}{2}$. Let $X\sim \GMM$ with 
	\begin{equation}
	R_{\min}> \sqrt{\frac{2}{1-2\lambda}\log\theta}
	\label{Rminproposition}
	\end{equation}  where
	$\theta$ is defined in \eqref{Pi_min_pi_max}. 
	Then for any $\mu\in \cal{U}_{\lambda}$   and all $j\ne i$, 
	with $c(\lambda)$ defined in \eqref{c(lambada)},  
	\begin{equation}
	\mathbb{E}_{i}[w_j(X,\mu)]\le\left(1+\frac{\pi_j}{\pi_i}\right)e^{-c(\lambda)R_{ij}^2 }.
	\label{eq:Eiwi}
	\end{equation}
\end{proposition} 

This proposition shows that $\Ei[w_j(X,\mu)]$ is {\em exponentially} small in the separation $R_{ij}$ and is key to proving contraction of the EM and gradient EM updates.
A similar result was proven in \citep{kwon2020algorithm}. 
The main differences are that they assumed a smaller region
with $\lambda<\frac1{16}$
and obtained a looser exponential bound  
$\exp(-R_{ij}^2/64)$. However, they considered a more challenging case where the weights $\pi_i$
and variances of the $K$ Gaussian components are unknown and are also estimated by the EM procedure. 

The key idea in proving Proposition \ref{Lem:w_exponentialy_small}  is that for $X\sim \mathcal{N}(\mu_i^*,I_d)$ it suffices to analyze the random variable $w_j(X,\mu)$  on the one dimensional space spanned by $\mu_i-\mu_j$. Thus, the expectation over a $d$ dimensional random vector is reduced to the expectation of some explicit function over a univariate standard Gaussian. 
%
An immediate corollary 
is  that  under the same conditions as in Proposition \ref{Lem:w_exponentialy_small}, 
the following lower bound holds for 
the expectation $\mathbb{E}_i[w_i(X,\mu)]$.

\begin{corollary}
	Set $0<\lambda < \frac{1}{2}$ and suppose that $R_{\min}$ satisfies \eqref{Rminproposition}. Then 
	$\forall \mu\in\cal{U}_{\lambda}$  
	\begin{equation}
	\mathbb{E}_{i}[w_i(X,\mu)]\ge 1- (K-1)(1+\theta )e^{-c(\lambda)R_i^2}. \label{Eq:w_exponentialy_close_to_1}
	\end{equation}  \label{Lem:w_exponentialy_close_to_1}
\end{corollary}
Next, note that for $X\sim\GMM$, it holds that $\mathbb{E}_X[w_i(X,\mu^*)]= \pi_i$. Thus, for center estimates $\mu$ close to $\mu^*$ we expect that $\mathbb{E}_X[w_i(X,\mu)]>\frac{3}{4}\pi_i$. This intuition is made precise in the following lemma which follows readily from Corollary \ref{Lem:w_exponentialy_close_to_1}.
\begin{lemma}
	Fix $0<\lambda<\frac{1}{2}$. Let $X\sim\GMM$ and suppose that 
	\begin{equation}
	R_{\min} \ge \sqrt{c(\lambda)^{-1}\log(15(K-1)(1+\theta ))}. \label{R_mindenombig}
	\end{equation}  
	Then for any $i\in [K]$ and any $\mu\in \cal{U}_{\lambda}$, 
	\begin{equation} \label{Eq:wi_greater}
	\mathbb{E}_X[w_i(X,\mu)] \ge \frac{3}{4}\pi_{i}.
	\end{equation} 
	\label{Lem:denom_big}
\end{lemma}

Next, we turn to the term $\E_X[\nabla_{\mu} w_i(X,\mu)(X-\mu_i)]$.
By definition, 
$\nabla_\mu w_i \in\mathbb{R}^{Kd}$ has the following $K$ components, each a vector in $\mathbb{R}^d$, 
\begin{equation}
\frac{\partial w_i(X,\mu)}{\partial \mu_i}= -w_i(X,\mu)(1-w_i(X,\mu))(\mu_i-X)   
\label{Eq:gradw_ii}
\end{equation}
and for $j\ne i$ 
\begin{equation}
\frac{\partial w_i(X,\mu)}{\partial \mu_j}= w_i(X,\mu)w_j(X,\mu)(\mu_j-X). \label{Eq:gradw_ij}
\end{equation}
For future use we introduce the following quantities
related to $\E_X[\nabla_{\mu} w_i(X,\mu)(X-\mu_i)]$.
For any $\mu,v \in \mathbb{R}^{Kd}$, 
define
\begin{align} \label{V_ij}
V_{i,j}(\mu,v)= \|\mathbb{E}_X[w_i(X,\mu)w_j(X,\mu)(X-v_i)(X-\mu_j)^{\top}]\|_{op},\\ \label{V_ii}
V_{i,i}(\mu,v)= \|\mathbb{E}_X[w_i(X,\mu)(1-w_i(X,\mu))(X-v_{i})(X-\mu_i)^{\top}]\|_{op}.
\end{align}
The following lemma, proved in the appendix, provides a bound on these quantities. 

\begin{lemma} \label{Lem:bound_expectation_grad}
	Fix $0<\lambda<\frac{1}{2}$. Let $X\sim\GMM$ with $R_{\min}$ satisfying Eq. (\ref{Rminproposition}). 
	Assume $\mu\in \cal{U}_{\lambda}$ and $v=\mu$ or $v=\mu^*$. Then, for any $i,j\in[K]$ with $i\neq j$
	\begin{align}
	\label{Viibound}   		
	V_{i,i}(\mu,v)  \le\sqrt{C(K-1)\left(1+\theta\right)}\max\left(d_{0},R_{i}^{2}\right)e^{-\frac{c\left(\lambda\right)}{2}R_{i}^{2}},
	\\
	\label{Vijbound}  V_{i,j}(\mu,v)\le\sqrt{C\left(1+\theta\right)}
	\max\left(d_{0},\max(R_{i},R_j)^2\right)
	e^{-\frac{c(\lambda)}2 \max(R_{i},R_j)^2 },
	\end{align} 
	where 
	$C$ is a universal constant, for example we can take $C=14$.
\end{lemma}
Expressions related to $V_{i,i}$ and $V_{i,j}$ were also studied by \cite{yan2017convergence}. They required a much larger separation, $R_{\min} \geq C \sqrt{d_0\log K}$, and their resulting bounds involved also $R_{\max}$.

\begin{remark}
	In proving the convergence of EM, the quantities of interest are $V_{i,j}(\mu,\mu^*)$ and $V_{i,i}(\mu,\mu^*)$, whereas for the gradient EM algorithm the relevant quantities are $V_{i,j}(\mu,\mu),V_{i,i}(\mu,\mu)$. 
	The reason for the effective dimension $d_0=\min(d,2K)$ is that for $d>2K$, in the population setting, the EM update of $\mu$ always remains in the subspace spanned by the $2K$ vectors $\{\mu_i\}_{i=1}^K$ and $\{\mu_i^*\}_{i=1}^K$. 	In the case of gradient EM, one may define a potentially smaller effective dimension $d_0=\min(d,K)$.  
\end{remark}

Last but not least, the following auxiliary lemma shows that
$\mu^*$ is a fixed point of the population EM update.
\begin{lemma} 
	\label{lem:popEMmustarfixed}
	Let $X\sim\GMM$. Then $\forall i\in [K]$, 
	$\mathbb{E}_X[w_i(X,\mu^*)(X-\mu_i^*)] =0 
	.
	$
\end{lemma} 

With all the pieces in place, we are now ready to prove Theorem \ref{main_EM_population_theorem}.

\begin{proof}[Proof of Theorem \ref{main_EM_population_theorem}]
	Consider a single EM update, as given by Eq. \eqref{populationEmupdate}, 
	\begin{equation} 
	\|\mu_i^+- \mu_i^*\|=
	\frac{1}{\mathbb{E}_X[w_i(X,\mu)]} \cdot 
	\left\|\mathbb{E}_X[w_i(X,\mu)(X-\mu_i^*)]\right\|, \quad \forall i\in[K]\nonumber
	\end{equation}
	Using Lemma \ref{lem:popEMmustarfixed}, we may write the numerator above as follows, 
	\begin{equation} \label{eq:numeratorww}
	\mathbb{E}_X[w_i(X,\mu)(X-\mu_i^*)]= \mathbb{E}_X[(w_i(X,\mu)- w_i(X,\mu^*))(X-\mu_i^*)].
	\end{equation} 
	By the mean value theorem there exists  $\mu^\tau$ on the line connecting $\mu$ and $\mu^*$ such that  
	\begin{equation}
	w_i(X,\mu)- w_i(X,\mu^*)=  \nabla_\mu w_i(X,\mu^{\tau})^\top(\mu- \mu^*). \label{Eq:w_igradinetgral}
	\end{equation}
	Inserting the expressions (\ref{Eq:gradw_ii}) and 
	(\ref{Eq:gradw_ij}) for the gradient of $w_i$
	into Eq. (\ref{Eq:w_igradinetgral}) gives
	\begin{align}
	w_i(X,\mu)- w_i(X,\mu^*) &=  w_i(X,\mu^\tau)(1-w_i(X,\mu^\tau))(X-\mu_i^\tau)^\top
	(\mu_i-\mu_i^*)  \nonumber \\
	& -\sum_{j\neq i}
	w_i(X,\mu^\tau)w_j(X,\mu^\tau)(X-\mu_j^\tau)^\top 
	(\mu_j-\mu_j^*).
	\nonumber
	\end{align}
	Taking expectations, and using the definitions of $V_{ii}$
	and $V_{ij}$, Eqs. (\ref{V_ij}) and (\ref{V_ii}), gives
	\begin{equation}
	\| \mathbb{E}[(w_i(X,\mu)- w_i(X,\mu^*))(X-\mu_i^*)] \|
	\leq 
	\sum_{j=1}^k  V_{ij}(\mu^\tau,\mu^*)\|\mu_j-\mu_j^*\| . 
	\label{eq:wimumu*}
	\end{equation}
	Since $\mu^{\tau}\in \U$ , we may apply Lemma \ref{Lem:bound_expectation_grad} to bound the terms on the right hand side above. Furthermore, 
	given that $x^2 e^{-tx^2}$ is monotonic decreasing for all $x>\sqrt{1/t}$ and $R_i\geq \sqrt{2/c(\lambda)}$, 
	we may replace all $R_i,R_j$ in the bounds of Lemma 
	\ref{Lem:bound_expectation_grad}	
	by $R_{\min}$.
	Defining $U = \frac{16\left(K-1\right)\sqrt{C\left(1+\theta\right)}}{3\pi_{\min}}$, we thus have
	\[
	\| \mathbb{E}_X[(w_i(X,\mu)- w_i(X,\mu^*))(X-\mu_i^*)] \|\leq
	\frac{3\pi_{\min}}{8} U \cdot  e^{\frac{-c\left(\lambda\right)}{2}R_{\min}^{2}} E(\mu).
	\]
	Next, note that condition \eqref{minimal_separation} on  $R_{\min}$ implies that it also satisfies the weaker condition (\ref{R_mindenombig}) of Lemma 
	\ref{Lem:denom_big}. Invoking this lemma yields that $\mathbb{E}_X[w_i(X,\mu)]\ge \frac{3\pi_{\min}}{4}$. 
	Thus, 
	\begin{equation*}
	\|\mu_i^+-\mu_i^*\| \le U \max\left(d_{0},R_{\min}^{2}\right)e^{\frac{-c\left(\lambda\right)}{2}R_{\min}^{2}}  \cdot \frac{E(\mu)}{2}.
	\end{equation*}
	If $d_0\geq R_{\min}^2$, then 
	for $E(\mu^+)\leq \frac12 E(\mu)$ to hold
	the minimal separation must satisfy
	\begin{equation}
	\frac{c(\lambda)}2 R_{\min}^2 \geq \log (d_0 U) .     
	\label{eq:bound_based_d0}
	\end{equation}
	In contrast, if $R_{\min}^2\geq d_0$ we obtain the following 
	inequality for $w = \frac{c(\lambda)}2 R_{\min}^2$, 
	\begin{equation}
	\label{eq:to_bound_based_R}
	w e^{-w}
	\leq \frac{c(\lambda)}{2U}  .
	\end{equation}
	Note that for $w>1$, the function $we^{-w}$ is monotonic decreasing. Also, consider the value
	$w^*=2\log(2U/c(\lambda))$ which is larger than 1, given the definitions of $U$ and of $c(\lambda)$.
	It is easy to show that $w^*\exp(-w^*) \leq c(\lambda)/2U$. Hence a sufficient condition for
	\eqref{eq:to_bound_based_R} to hold is that $w>w^*$, namely 
	\begin{equation}
	\frac{c(\lambda)}2 R_{\min}^2 \ge 2\log \frac{2U}{c(\lambda)}.
	\label{eq:bound_based_Rmin}
	\end{equation}
	It is easy to verify that $\log U+\log(4/c(\lambda))>\log d_0$ and thus
	the bound of (\ref{eq:bound_based_Rmin}) is more restrictive than (\ref{eq:bound_based_d0}). 
	Inserting the expression for $U$ into Eq. (\ref{eq:bound_based_Rmin}) yields
	the condition of the theorem, Eq. (\ref{minimal_separation}).
	Finally, to complete the proof we need to show that  for all $i$,  $\|\mu_i^+-\mu_i^*\|\le \lambda R_i$. 
	This part is proven in auxiliary lemma \ref{lem:mu_plus_mu_i} in the appendix. 
\end{proof}

\section{PROOF FOR THE SAMPLE EM} \label{section:sample}

In this section we prove our results on the sample EM and gradient EM algorithms. The main idea is to show  concentration results for both the denominator and  the numerator of the EM update. Our strategy is similar to \cite{zhao2020statistical} 
but with several improvements. 
First,  our result on the concentration of the denominator of the EM update, Lemma \ref{w_i_conc}, only considers samples from the $i$-th cluster. Thus, in Lemma \ref{sample_denom_lemma},  we obtain a uniform lower bound for the weight $w_i$ with $n=\tilde{\Omega}(Kd/\pi_{\min})$ compared to the larger   $n=\tilde{\Omega}(Kd/\pi_{\min}^2)$ in \cite{zhao2020statistical}. 
Second, while \cite{zhao2020statistical} bounded the sub-Gaussian norm of the numerator of the EM update by $CR_{\max}$, 
we derive in Lemma \ref{Lemm:w_i(x-mu)sub-Gaussnorm} a tighter bound, which does not depend on $R_{\max}$. This in turn, yields a tighter concentration for the numerator of the EM update in Lemma \ref{lem:concentration_numerator}.


\begin{lemma}
\label{w_i_conc}
	Fix $\delta\in (0,1), \lambda\in (0,\frac12)$  and let $X_1, \ldots, X_{n_i} \overset{i.i.d.}{\sim} \mathcal{N}(\mu_i^*,I_d)$ . Then with probability at least $1-\delta$, 
	\begin{equation}
	\sup_{\mu\in \mathcal{U}_{\lambda}}\left| \frac{1}{n_i}\sum_{\ell=1}^{n_i}w_i(X_{\ell},\mu)-\mathbb{E}_i[w_i(X,\mu)] \right| \le  \sqrt{\tilde{c}\frac{Kd\log\left(\frac{\tilde{C} n_i}{\delta}\right)}{n_i}}
	\label{eq:w_i_conc}
	\end{equation}  
	where $\tilde{C}= 18K(\sqrt d+2R_{\max})R_{\max}$ and $\tilde{c}$ is a suitable universal constant.
\end{lemma}  

As we saw in Lemma \ref{Lem:denom_big}, 
the denominator in the population EM update for the $i$-th mean is lower bounded by $\frac34\pi_i$.  We use Lemma \ref{w_i_conc} to show that this lower bound holds also for the finite sample case. 
We remark that a version of the following lemma appeared in \cite{zhao2020statistical}, but with a larger sample size requirement of  $n= \tilde{\Omega}(Kd/\pi_{\min}^2)$.

\begin{lemma} \label{sample_denom_lemma}
	Fix $\delta\in(0,1),\lambda\in(0, \frac12)$. Let $X_1, \ldots, X_n \overset{i.i.d.}{\sim} \GMM$, with $R_{\min}$ that satisfies \eqref{R_mindenombig}.
	Assume a sufficiently large sample size  $n$ such that
	\begin{equation}
	\frac{n}{\log n} >C\frac{Kd\log\frac{\tilde{C}}{\delta}}{\pi_{\min}} \label{eq:sample_requirement_denom}
	\end{equation}
	where $\tilde{C}= 100K^2\pi_{\max}(\sqrt{d}+2R_{\max}) R_{\max}$ and $C$ is a universal constant.  
	For any $i\in[K]$, define the event  
	\begin{equation}
	D_i= \left\{\inf_{\mu\in \mathcal{U}_{\lambda}} \frac{1}{n}\sum_{\ell=1}^nw_i(X_{\ell},\mu)  \ge \frac{3\pi_i}{4}\right\} .
	\label{sample_denom_lemma_event}
	\end{equation}
	Then, the event $D_i$ occurs with probability at least $1-\frac{\delta}{2K}$.
\end{lemma}

Next, we analyze the sub-Gaussian norm of $w_i(X,\mu)(X-\mu_i^*)$. 
 \cite{zhao2020statistical} bounded this quantity by  $CR_{\max}$.  We present an improved bound which does not depend on $R_{\max}$.  
 For the definition of the sub-Gaussian 
 norm $\|\cdot \|_{\psi_2}$, see the Appendix.

\begin{lemma} \label{Lemm:w_i(x-mu)sub-Gaussnorm}
	Fix $\lambda\in(0,\frac12)$. Let $X\sim \GMM$ with
	\begin{equation} 
	\label{eq:minimal_sep_gauss_norm}
	R_{\min}\ge\sqrt{\max\left(\frac{4}{1-2\lambda}\log\left(4\log(\tfrac32)\theta^{2} 
		\frac{1-2\lambda}{c(\lambda)}\right), \frac{4}{c(\lambda)}\log 2  \right)}.
	\end{equation}
	Suppose that $\mu\in\U$. Then for any $i\in [K]$, 
	\begin{equation}
	\|w_{i}\left(X,\mu\right)\left(X-\mu_i^*\right)\|_{\psi_{2}}\le \frac{16}{1-2\lambda} \label{eq:w_i(x-mu^*)sub-Gaussnorm}
	\end{equation}
	and  
	\begin{equation}
	\|w_{i}\left(X,\mu\right)\left(X-\mu_i\right)\|_{\psi_{2}}\le 24  \max\left( \frac{1}{1-2\lambda},  \lambda R_i \right). \label{eq:w_i(x-mu)sub-Gaussnorm}
	\end{equation} 
\end{lemma}

Using Lemma \ref{Lemm:w_i(x-mu)sub-Gaussnorm} we upper bound the concentration of  the numerator 
in the expression for the error in  the  sample EM update, Eq. (\ref{eq:sample_em_update}).

\begin{lemma} \label{lem:concentration_numerator} 
	Fix $\delta\in (0,1),\lambda\in (0,\frac12)$. Let $X_1, \ldots, X_n \overset{i.i.d.}{\sim}\GMM$ with $R_{\min}$ satisfying \eqref{eq:minimal_sep_gauss_norm}. 
	For $i\in [K]$ define $S_i=\frac1n \sum_{\ell=1}^nw_i(X_{\ell},\mu)(X_{\ell}-\mu_i^*)$ and the event
	\begin{equation}
	N_i =\left\{\sup_{\mu\in \mathcal{U}_{\lambda}} \left\|S_i- \mathbb{E}_X[w_i(X,\mu)(X-\mu_i^*)] \right\| \le \frac{C}{1-2\lambda}\sqrt{\frac{Kd\log\frac{\tilde{C}n}{\delta}}{n}}\right\} \label{sample_num_lemma_event}
	\end{equation}
	Then, with $\tilde{C}= 36K^2R_{\max}(\sqrt d+2R_{\max})^2$ and with a suitable choice of a universal constant $C$, the event $N_i$ occurs with probability at least $1-\frac{\delta}{2K}$.
\end{lemma}

With all the pieces in place, we are now ready to prove Theorem \ref{sample_em_main_theorem}.

\begin{proof}[Proof of Theorem \ref{sample_em_main_theorem}]
	Consider the error of a single the update of the from \eqref{eq:sample_em_update} of the sample EM algorithm,
	\begin{equation*}
	\|\mu_i^+-\mu_i^*\| = \frac{\|\frac1n\sum_{\ell=1}^nw_i(X_{\ell},\mu)(X_{\ell}-\mu_i^*)\|}{\frac1n\sum_{\ell=1}^nw_i(X_{\ell},\mu)} \nonumber. 
	\end{equation*}
	Note that the requirement \eqref{minimal_separation} on $R_{\min}$ is more restrictive than  \eqref{R_mindenombig}. Also, the  sample size requirement \eqref{minimal_sample} is more restrictive than  \eqref{eq:sample_requirement_denom}. Thus, we may invoke Lemma \ref{sample_denom_lemma} and get that with probability at least $1-\frac{\delta}{2K}$, that event $D_i$ \eqref{sample_denom_lemma_event} occurs. 
	Hence, 
	\begin{align}
	\|\mu_i^+-\mu_i^*\| 
	& \le  \frac4{3\pi_i} \left\|S_i-\mathbb{E}_X[w_i(X,\mu)(X-\mu_i^*)]\right\|  + \frac4{3\pi_i} \left\|\mathbb{E}_X[w_i(X,\mu)(X-\mu_i^*)]\right\| 
	\nonumber
	\end{align}
	It follows from Theorem \ref{main_EM_population_theorem} that for $R_{\min}$ satisfying \eqref{minimal_separation}, the second term above is upper bounded by $ \frac12 \min( E(\mu),\lambda R_i)$. We thus continue by bounding the first term above.
	Note that our requirements on the minimal separation \eqref{minimal_separation} is more restrictive than the requirement in \eqref{eq:minimal_sep_gauss_norm}. Thus, we may invoke  Lemma \ref{lem:concentration_numerator} and obtain with probability at least $1-\frac{\delta}{2K}$ , that the event $N_i$ \eqref{sample_num_lemma_event} occurs.  Therefore, 
	\begin{equation}
	\|\mu_i^+-\mu_i^*\| \le \frac12 \min( E(\mu), \lambda R_i) + 	   \frac{C}{(1-2\lambda)\pi_i}\sqrt{\frac{Kd\log\frac{\tilde{C}n}{\delta}}{n}} \label{eq:finsampleerr}
	\end{equation}
	where  $C$ is a universal constant and $\tilde{C}= 36K^2R_{\max}(\sqrt d+2R_{\max})^2$.
	For $n$ sufficiently large  so that \eqref{minimal_sample} is satisfied, it holds that 
	$C\frac{1}{(1-2\lambda)\pi_i}\sqrt{\frac{Kd\log\frac{\tilde{C}n}{\delta}}{n}} \le \frac12\lambda R_i$
	and therefore $\|\mu_i^+-\mu_i^*\| \le \lambda R_i$. 
	By a union bound over all $i\in[K]$, with
	probability at least $1-\delta$, $\mu^+\in \U$. 
     This allows us to iteratively apply  \eqref{eq:finsampleerr} and obtain Eq. \eqref{convergence_up_to_statistical_errormain}. 
\end{proof}

\section{Simulations}
\label{sec:simulation}

\begin{figure}[h!]
    \centering
    \begin{subfigure}[b]{0.5\columnwidth}
      \includegraphics[width = \columnwidth ,keepaspectratio]{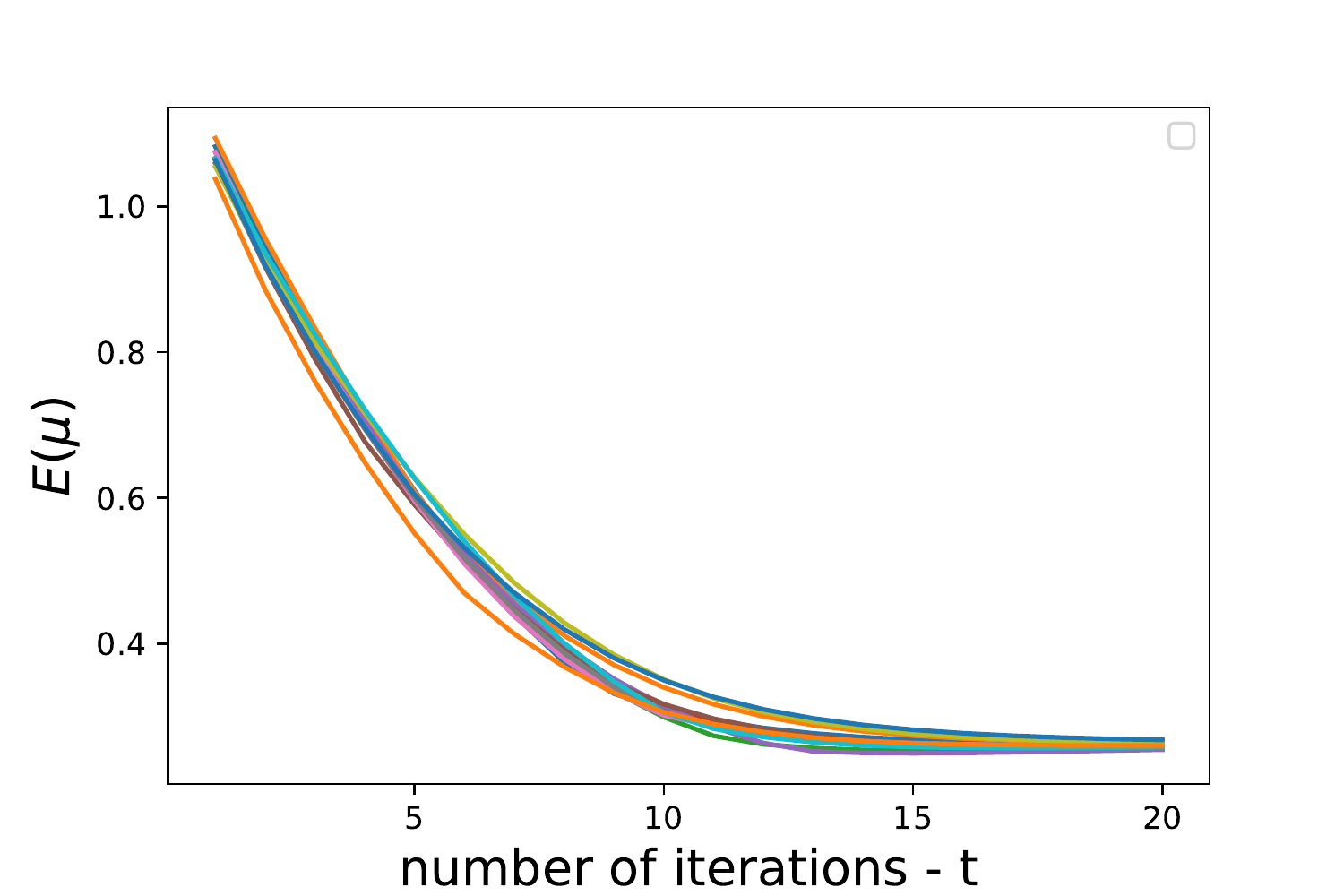}
        \caption{}
        \label{figa}
    \end{subfigure}%
       \begin{subfigure}[b]{0.5\columnwidth}
      \includegraphics[width = \columnwidth ,keepaspectratio]{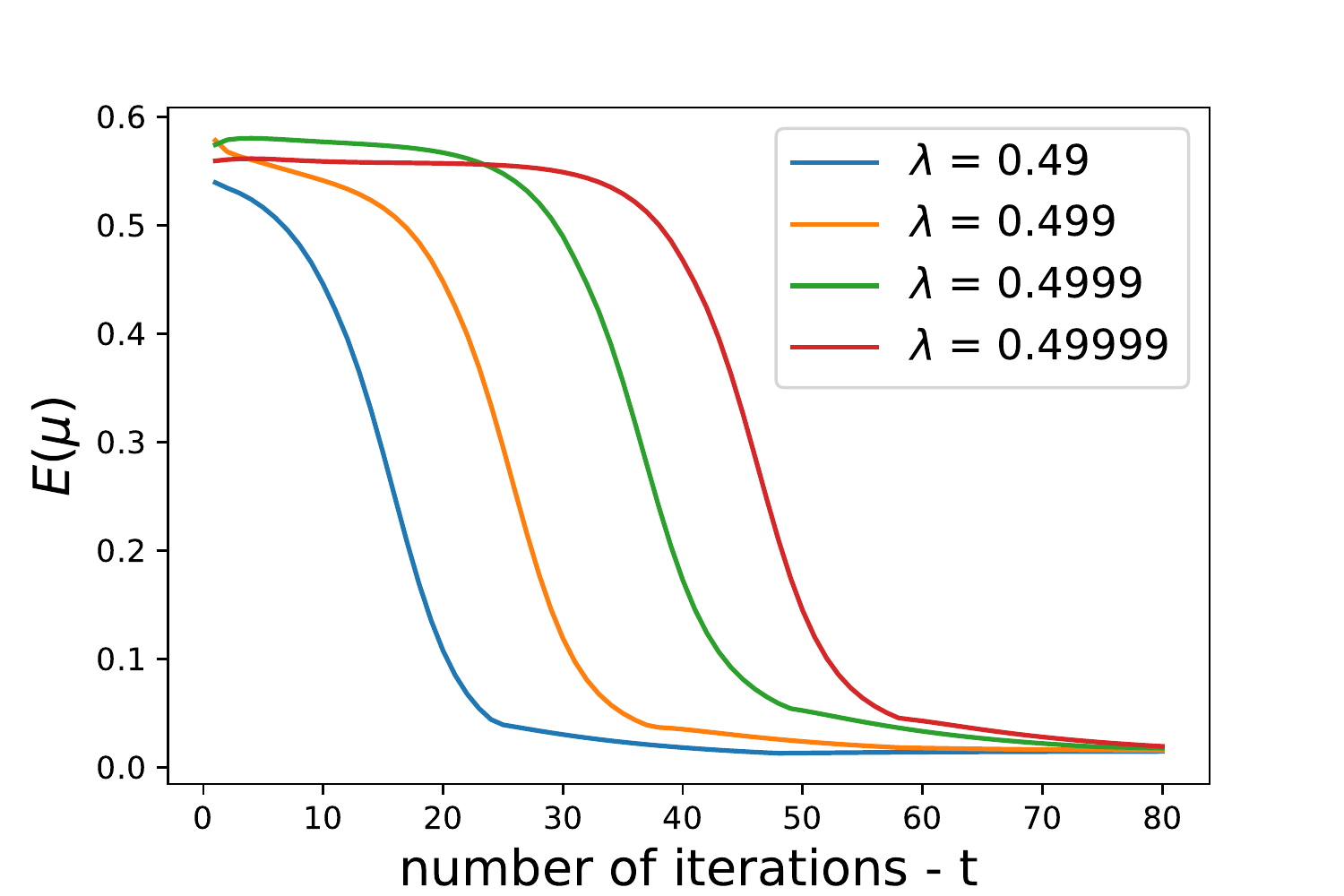}
        \caption{}
        \label{figb}
    \end{subfigure}%
    \hfill
    \begin{subfigure}[b]{0.5\columnwidth}
      \includegraphics[width = \columnwidth ,keepaspectratio]{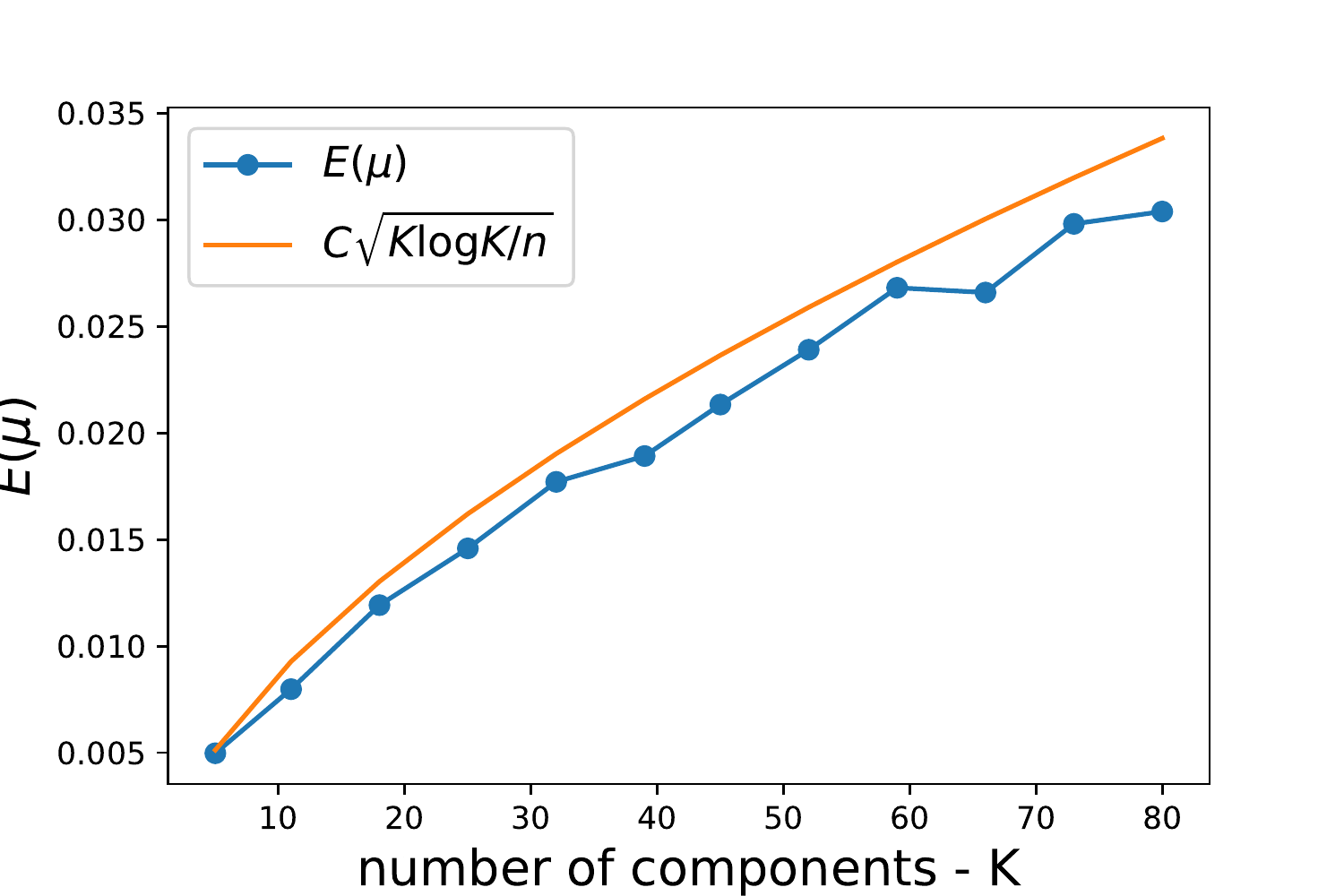}
        \caption{}
        \label{figc}
    \end{subfigure}%
        \begin{subfigure}[b]{0.5\columnwidth}
      \includegraphics[width = \columnwidth ,keepaspectratio]{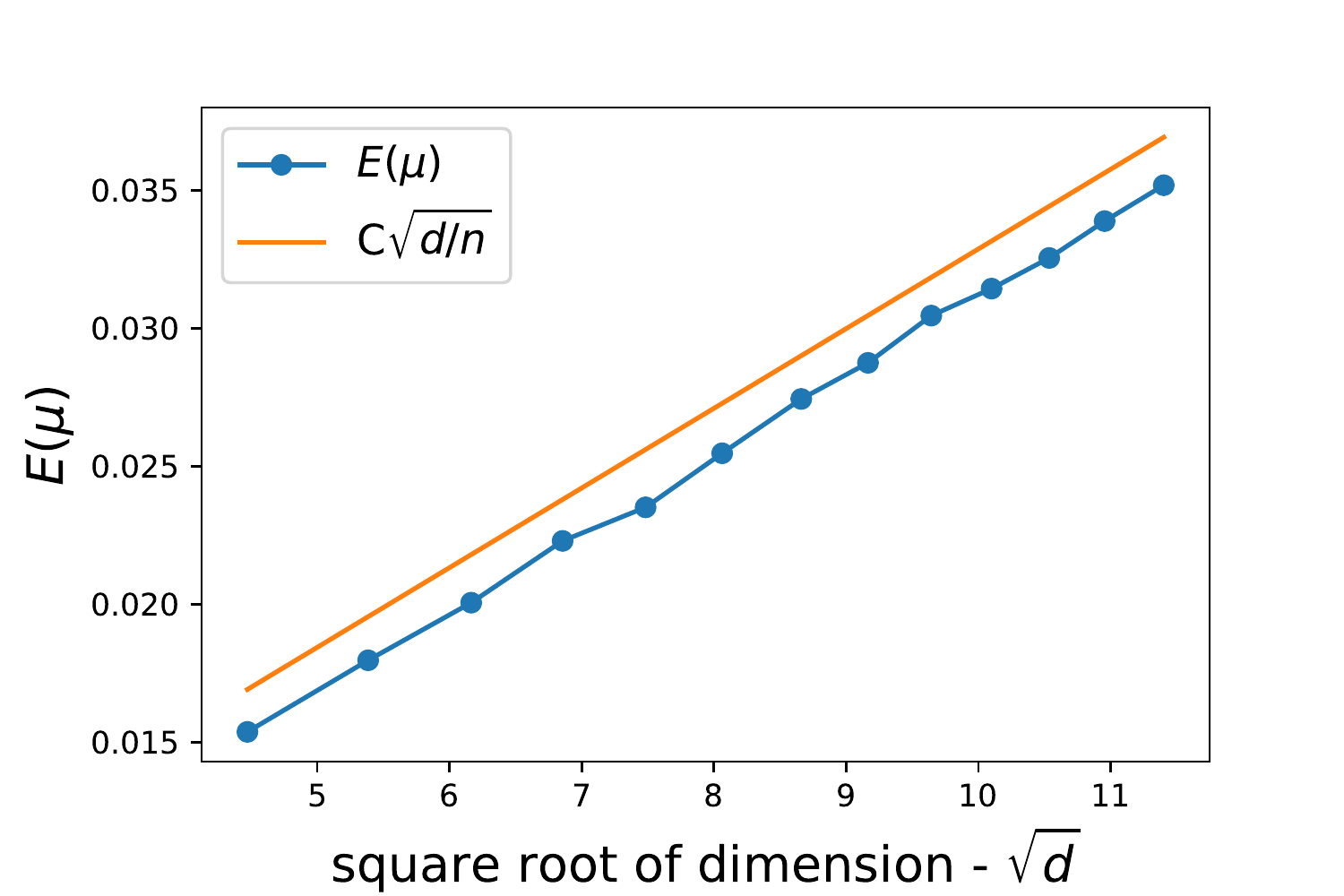}
        \caption{}
        \label{figd}
    \end{subfigure}%
    \caption{Top left: Convergence of  EM  for a GMM  with $64$ components in $\mathbb R^{64}$.  Each line is the error starting from a different random initialization.  Top right: Convergence of EM with initial estimates $\|\mu_i-\mu_i^*\|= \lambda R_i$ and $\lambda$ slightly smaller and than $\frac12$ for a $5$ component $10$ dimensional GMM.
	Bottom left: The error as a functions of the number of components for $1$ dimensional GMMs averaged over 25 runs. The error behaves like a constant $C$ times $\sqrt{\frac{K\log K}{n}}$.
	Bottom right: The error as a functions of the dimension for $5$ component GMMs with   averaged over 25 runs. The error behaves like a constant $C$ times $\sqrt{\frac{d}{n}}$. } 
	\label{fig:sim}
\end{figure}

We present numerical  simulations  with the EM algorithm
and compare them to our theoretical results.
The ability of the EM and gradient EM algorithms to learn Gaussian mixture models has been extensively demonstrated in simulations by various authors, see \cite{zhao2020statistical,yan2017convergence} and references therein.  Our simulations focus on several  quantities that appear in Theorem \ref{main_EM_population_theorem}.
First, we demonstrate that even in a 
setting with a moderately 
high dimension and with a large number of components, a relatively low separation suffices for the EM algorithm to yield accurate estimates. Unlike \cite{zhao2020statistical}, which presented numerical results for a $5$ component GMM in $\mathbb R^{10}$, we  consider a 64 component GMM with centers on the unit simplex in $\mathbb R^{64}$. We generate $5\cdot10^5$ samples from this GMM and 
consider several initializations where each initial estimate $\mu_i$ is sampled uniformly from a sphere of radius $0.45R_{i}$ around $\mu_i^*$. 
In Figure \ref{figa}  we plot the error $E(\mu)$ as a function of the number of iterations for 12 random initializations. We see that the EM algorithm yields accurate estimates  in this setting,  even though the separation between the different Gaussians is small relative to the dimension and to number of components.

Next, we  explore the effect of the constant $\lambda$  such that the initial estimates satisfy $\|\mu_i-\mu_i^*\|\le \lambda R_i$ for values of $\lambda$ slightly smaller than $\frac{1}{2}$. We consider a $5$ components GMM with centers on the unit simplex in $\mathbb{R}^{10}$. We generate $5\cdot10^5$ samples  and run the EM algorithm. The initial values $\mu_1$ and $\mu_2$ are chosen on the line connecting $\mu_1^*$ and $\mu_2^*$. The other  $3$  initial value $\mu_3,\mu_4,\mu_5$ are sampled uniformly from a unit sphere of radius $\lambda R_i$ and center $\mu_i^*$. As can be seen in Figure \ref{figb}, the EM algorithm yields accurate estimates for all considered values of $\lambda$ smaller than $\frac12$, even when $\lambda=\frac{1}{2}-10^{-5}$.

Finally, we consider the accuracy of EM for GMMs as we increase either
the number of components or the dimension. Specifically, we considered a $1$ dimensional GMM with $K$ equally spaced components with $R_{\min}=10$ and varying values of $K$. We generate $5\cdot10^5$ samples from each GMM and run the EM algorithm for $20$ iterations. Fig. \ref{figc} shows the error, averaged over $25$ runs as a function of $K$. As seen in the plot the error behaves like $\sqrt{K\log(K)/n}$, which is also the expected parametric error if all samples were labeled, which means that on average we had  $n/K$ samples from each component.
These results suggest that the upper bound of Eq. \eqref{convergence_up_to_statistical_errormain}  which depends on $\sqrt{K^3}$ may be improved. 
Next, we considered a sequence of GMMs
in increasing dimension $\mathbb{R}^d$ where $d\in [20,130]$. 
Each GMM had $5$ components with centers $Re_i$ for $1\le i\le 5$, where $R=10$ and $e_i$ is the standard Euclidean basis vector in $\mathbb{R}^d$. We generate $5\cdot10^5$ samples from each GMM and run the EM algorithm for $20$ iterations. We plot the error averaged over $25$ runs as a function of the square root of the dimension $d$. 
In accordance to Eq.  \eqref{convergence_up_to_statistical_errormain},  the empirical error seems 
to increase like $\sqrt{d}$.

\appendix

\section{PROOFS FOR SECTION \ref{section:population}} \label{appendix:population}


\subsection{Proof of Proposition \ref{Lem:w_exponentialy_small}}

Before proving Proposition \ref{Lem:w_exponentialy_small} we state several auxiliary lemmas.
\begin{lemma} \label{Lemm:monotonic}
	Let $g(A,B)$ be the following function of two variables, 
	\begin{equation}
	g(A,B) = \int \frac1{\sqrt{2\pi}} \frac1{1+\alpha e^{At+B}} e^{-t^2/2}dt
	\nonumber	\end{equation}
	where $\alpha>0 $ is a fixed constant. Then: (i) For any fixed $A$, $g(A,B)$ is monotonic decreasing in $B$; and (ii)
If in addition $\alpha>e^{-B} $ and $A>0$, then for any fixed $B$, $g(A,B)$ is monotonic increasing in $A$.
\end{lemma}
\begin{proof}
	Since the function inside the integral is monotonically decreasing in $B$, 
	part (i) directly follows. To prove part (ii), we take the derivative with respect to $A$,
	\begin{equation}
	\frac{\partial }{\partial A}g(A,B) = \int  \frac{-\alpha t e^{At+B} }{(1+\alpha e^{At+B})^2}   \frac{e^{-t^2/2}}{\sqrt{2\pi}} dt.\nonumber
	\end{equation}
	Denote the function inside the integral by $f(t)$. Note that  $f(t)>0$ when $t<0$ and $f(t)<0$ when $t>0$. To show that the integral is positive it suffices to show that for all $t>0$ it holds that $-f(t)<f(-t)$. This condition reads as  \begin{equation}
	\frac{e^{-At}}{\left(1+\alpha e^{-At+B}\right)^{2}}>\frac{e^{At}}{\left(1+\alpha e^{At+B}\right)^{2}}. \nonumber
	\end{equation} 
	Some algebraic manipulations give that this condition is equivalent to    \begin{equation}
	\left(e^{At}-e^{-At}\right)\left(\alpha^{2}e^{2B}-1\right)>0\nonumber
	\end{equation} 
	which is indeed satisfied for $A,t>0$ and $\alpha>e^{-B}$.
\end{proof}

\begin{lemma} \label{Lemm:ABbounds}
	Fix any two distinct vectors $\mu_i^*, \mu_j^*\in \mathbb{R}^d$ and $\lambda\in(0,1/2)$. Denote the ball of radius $r$ about the origin in $\mathbb{R}^d$ by $B_d(0,r)$ and define \begin{equation}
	\Omega= B_d(0,\lambda\|\mu_i^*-\mu_j^*\|)\times B_d(0,\lambda\|\mu_i^*-\mu_j^*\|) \subset \mathbb{R}^d\times \mathbb{R}^d. \nonumber
	\end{equation}   Consider the two functions  $A,B: \Omega \to \mathbb{R}$ \begin{align}
	A(\xi_i, \xi_j)= \|\mu_i^*-\xi_i-\mu_j^*+\xi_j\| \label{DefA}\\
	B(\xi_i,\xi_j)= \frac{1}{2} \|\mu_{i}^{*}-\mu_{j}^{*}+\xi_{j}\|^{2}-\frac{1}{2}\|\xi_{i}\|^{2}.\label{DefB}
	\end{align}
	Then 
	for any $(\xi_i, \xi_j)\in \Omega$, \begin{align}
	A(\xi_i,\xi_j) \le (1+2\lambda)\|\mu_i^*-\mu_j^*\|= A^* 
	\label{A}
	\\B(\xi_i,\xi_j)\ge \frac{1-2\lambda}{2}\|\mu_i^*-\mu_j^*\|^2= B^* \label{B}. 
	\end{align}
\end{lemma}
\begin{proof}
	We first prove the upper bound on $A$. By the triangle inequality
	\begin{equation}
	A(\xi_i,\xi_j)
	\le \|\xi_i\|+\|\xi_j\| + \|\mu_i^*-\mu_j^*\| \le (1+2\lambda)\|\mu_i^*-\mu_j^*\| \nonumber
	\end{equation}
	As for the lower bound on $B$, clearly it is obtained when $\|\xi_i\|$ is maximal, i.e. $\|\xi_i\| =\lambda\|\mu_i^*-\mu_j^*\|$. Finally, the vector $\xi_j=\lambda(\mu_i^*-\mu_j^*)$ minimizes \eqref{DefB} regardless of the value of $\xi_i$. This yields the lower bound of \eqref{B} for $B$. 
\end{proof}

\begin{proof}[Proof of Proposition \ref{Lem:w_exponentialy_small}]
	
	Recall the definition of the weight $w_j(X,\mu)$ in Eq \eqref{Eq:w}. 
	Since all the terms in the denominator are positive, we may upper bound $w_j$ by taking into account only the two terms with indices $k=i$ and $k=j$. 
	Hence,
	\begin{equation}
	w_j(X,\mu) \leq \frac{\pi_j e^{-\frac{\|X-\mu_j\|^2}{2}}}
	{\pi_j e^{-\frac{ \|X-\mu_j\|^2}{2}} + \pi_i  e^{ - \frac{\|X-\mu_i\|^2}{2}}}
	= 
	\frac{1}{1+\frac{\pi_i}{\pi_j}e^{\frac{\|X-\mu_j\|^2}{2} -\frac{\|X-\mu_i\|^2}{2}}}.  \label{wi1d}
	\end{equation}
	Next, since $X\sim \mathcal{N}(\mu_i^*, I_d)$ we may write $X=\mu_i^*+\eta=\mu_i+\eta+\xi_i$ where $\eta\sim \mathcal{N}(0,I_d)$ and $\xi_i= \mu_i^*- \mu_i$. Therefore, \begin{align}
	\|X-\mu_{j}\|^{2}-\|X-\mu_{i}\|^{2}	&= 2\eta^{\top}\left(\mu_{i}-\mu_{j}\right)+\|\mu_{i}^{*}-\mu_{j}^{*}+\xi_{j}\|^{2}-\|\xi_{i}\|^{2}. \label{wi1d2}
	\end{align}
	Note that by definition
	$\eta^{\top} (\mu_i-\mu_j)$ is a univariate Gaussian random variable with mean zero
	and variance $\|\mu_i-\mu_j\|^2$. Hence, we may write $\eta^{\top} (\mu_i-\mu_j) = \|\mu_i-\mu_j\|\nu$  where $\nu\sim \mathcal{N}(0,1)$. Defining $\tilde{w}(A,B,\nu)=1/(1+\frac{\pi_i}{\pi_j}e^{A\nu+B})$, we therefore have 
\begin{align}
	\mathbb{E}_{i}[w_j(X,\mu)] 
	&\le
	\mathbb{E}_{\nu} \left[\tilde{w}(A,B,\nu)\right]=\frac1{\sqrt{2\pi}}\intop\tilde{w}(A,B,t) e^{-\frac{t^2}{2}} dt = g(A,B)
	\label{1dimexpectation}
\end{align}
	with $A=A(\xi_i, \xi_j)$ and $B=B(\xi_i, \xi_j)$ as defined  in \eqref{DefA}  and \eqref{DefB}, respectively. 
	Since $\|\xi_i\|,\|\xi_j\|\le \lambda\|\mu_i^*-\mu_j^*\|$, then 
	$A\ge (1-2\lambda)\|\mu_i^*-\mu_j^*\|$. 
	Therefore, $A>0$ for $\lambda<\frac{1}{2}$. 
	By Lemma \ref{Lemm:ABbounds}, $B\geq B^*$
    with $B^*$ given 
in \eqref{B}. 
    The condition (\ref{Rminproposition}) implies that $\frac{\pi_i}{\pi_j} > e^{-B^*}\geq e^{-B}$. Hence, the conditions of Lemma \ref{Lemm:monotonic} are satisfied and we can upper bound 
    $g(A,B)$ in \eqref{1dimexpectation}, by $g(A^*,B^*)$ with $A^*$ and $B^*$ respectively, as given in Equations \eqref{A} and \eqref{B} of Lemma \ref{Lemm:ABbounds}. Therefore, 
	\begin{equation}
	\mathbb{E}_{i}[w_j(X,\mu)] \le \frac1{\sqrt{2\pi}}\intop \tilde{w} (A^*, B^*,t)e^{-\frac{t^2}{2}} dt=I. 
	\nonumber
	\end{equation}
	To upper bound the integral $I$ we split it into two parts based on the sign of $A^*t+B^*$.
	\begin{equation}
	I = \frac1{\sqrt{2\pi}}\intop_{-\infty}^{-B^*/A^*} \tilde{w}(A^*,B^*,t) e^{-\frac{t^2}{2}}dt 
	+
	\frac1{\sqrt{2\pi}}\intop_{-B^*/A^*}^{\infty} \tilde{w}(A^*,B^*,t)e^{-\frac{t^2}{2}}dt  =I_1+I_2. \nonumber
	\end{equation}
	For $I_1$, where $A^*t+B^*<0$, we upper bound $\tilde w(A^*,B^*,t)\leq 1$. Since both $A^*$ and $B^*$ are positive, we have that $-\frac{B^*}{A^*}<0$. We can therefore use Chernoff's bound to get \begin{equation}
	I_1\le  \intop_{-\infty}^{-B^*/A^*} \frac{1}{\sqrt{2\pi}} e^{-\frac{t^2}{2}}dt \le e^{-\frac{1}{2}\left( \frac{B^*}{A^*}\right)^2}. \label{eq:bound_I1}
	\end{equation}
	For $I_2$, where $A^*t+B^*>0$ we upper bound the integral by ignoring the constant $1$ in the denominator. Completing the square and changing variables by $z=t+A^*$ we get \begin{align}
    	I_2 &\le \intop_{-B^*/A^*}^{\infty} \frac{1}{\sqrt{2\pi}} \frac{\pi_j}{\pi_i}e^{-\frac{t^2}{2}-A^*t-B^*}dt \nonumber \\&= e^{ \frac{A^{*2}}{2}-B^*}\intop_{-B^*/A^*}^{\infty} \frac{1}{\sqrt{2\pi}} \frac{\pi_j}{\pi_i}e^{-\frac{(t+A^*)^2}{2}}dt= e^{ \frac{A^{*2}}{2}-B^*}\intop_{A^*-B^*/A^*}^{\infty} \frac{1}{\sqrt{2\pi}} \frac{\pi_j}{\pi_i}e^{-\frac{z^2}{2}}dz. \nonumber
	\end{align}
	Using the definitions of $A^*$ and $B^*$ in \eqref{A} and \eqref{B} we note that for $\lambda>0$, 
	\begin{equation}
	A^*-\frac{B^*}{A^*}= \frac{2\left(1+2\lambda\right)^{2}-\left(1-2\lambda\right)}{2(1+2\lambda)} \|\mu_{i}^{*}-\mu_{j}^{*}\| >0.
	\nonumber
	\end{equation}
	We can therefore apply Chernoff's bound on the above and obtain \begin{equation}
	I_2 \le \frac{\pi_j}{\pi_i}e^{ \frac{A^{*2}}{2}-B^*- \frac{1}{2}\left(A^*-\frac{B^*}{A^*}\right)^2}=  \frac{\pi_j}{\pi_i}e^{ -\frac12 \left(\frac{B^*}{A^*} \right)^2}.\label{eq:bound_I2}
	\end{equation}
	Combining the two bounds (\ref{eq:bound_I1}) and (\ref{eq:bound_I2}) yields Eq \eqref{eq:Eiwi}.
	
\end{proof}

\begin{proof} [Proof of Corollary \ref{Lem:w_exponentialy_close_to_1}]
	By definition, the sum of all weights is one. Thus, 
\begin{equation}
	w_i(X,\mu)=1-\sum_{j\ne i}w_j(X,\mu) . \nonumber
	\end{equation}  
	By Proposition \ref{Lem:w_exponentialy_small}  and the linearity of expectation \begin{equation}
	\mathbb{E}_i[w_i(X,\mu)]\ge 1-
	\sum_{j\ne i}\left(1 + \frac{\pi_j}{\pi_i}\right)e^{-c(\lambda)R_{ij}^2} \ge 1 -(K-1)(1+\theta)e^{-c(\lambda)R_i^2}.  \nonumber
	\end{equation}
\end{proof}

\subsection{Proof of Lemma \ref{Lem:denom_big}}
\begin{proof}
	Since $X$ is distributed as a GMM with $K$ components and $w_i(X,\mu)>0$, the expected value is greater than if we  consider only the $i$-th component of the GMM.
	\begin{align*} 
	\mathbb{E}[w_i(X,\mu)]&= \sum_{j=1}^{K}\pi_j\mathbb{E}_{j}[w_i(X,\mu)]
	\ge  \pi_i\mathbb{E}_{i}[w_i(X,\mu)].
	\end{align*}
	Since the requirement (\ref{R_mindenombig}) on $R_{\min}$  implies \eqref{Rminproposition}, it follows from Corollary \ref{Lem:w_exponentialy_close_to_1} that \begin{equation} 
	\mathbb{E}_X[w_i(X,\mu)]\ge \pi_i \left(1- (K-1)(1+\theta)  )e^{-c(\lambda)R_i^2}\right).
	\nonumber
	\end{equation}
	Furthermore, Eq. (\ref{R_mindenombig}) implies that $
	(K-1)(1+\theta )e^{-c(\lambda)R_i^2} \le \frac{1}{4} $. 
\end{proof}

\subsection{Proof of Lemma \ref{Lem:bound_expectation_grad}} 

The proof consists of several steps. First, in Lemma  \ref{reduce_dim_lemma} we reduce the dimension to $d_0=\min(d,2K)$. Next, in Lemma \ref{Lemma:normbounds} we bound $\|\mathbb{E}_{k}[(X-v_i)(X-\mu_j)^\top]\|_{op}$ in terms of  $d_0,R_i,R_j$ and $R_ij$. We then present the proof of the Lemma.

We first  introduce notations. For  $\mu,v\in \mathbb{R}^{Kd}$ and  $i,j,k\in [K]$ with $i\ne j$ we define,
  \begin{align}
     V_{ij}^k(\mu,v)=\|\mathbb{E}_k[w_i(X,\mu)w_j(X,\mu) (X-v_i)(x-\mu_j)^\top ]\|_{op}
     \label{Vijk}
     \\
     V_{ii}^k(\mu,v)=\|\mathbb{E}_k[w_i(X,\mu)(1-w_i(X,\mu))(X-v_i)(x-\mu_i)^\top]\|_{op}. \label{Viik}
\end{align}

Suppose that $X\sim \mathcal{N}(\mu_k^*,I_d)$.  Let $\Gamma$ be a rotation matrix such that $(\Gamma\mu_{i})^\top=(\overline{\mu_{i}}^{\top},0_{\left[d-d_{0}\right]_{+}}^{\top})$ and $\left(\Gamma v_i\right)^\top=(\overline{v_{i}}^{\top},0_{\left[d-d_{0}\right]_{+}}^{\top})$, 
	for all $i\in [K]$.  Write $\overline{X}^{d_{0}}$ for the first  $d_{0}$ coordinates of $\Gamma X$ and $\overline{X}^{d-d_{0}}$ for the remaining coordinates. We define 
\begin{align}
     \overline{V}_{ij}^k(\mu,v)=
\|\mathbb{E}_{k}[w_{i}(\overline{X}^{d_{0}},\overline{\mu})(w_{j}(\overline{X}^{d_{0}},\overline{\mu}))(\overline{X}^{d_{0}}-\overline{v_{i}})(\overline{X}^{d_{0}}-\overline{\mu_{j}})^{\top}]\|_{op} 
     \label{Vijkd_0}
     \\
\overline{V}_{ii}^{k}(\mu,v)=\|\mathbb{E}_{k}[w_{i}(\overline{X}^{d_{0}},\overline{\mu})(1-w_{i}(\overline{X}^{d_{0}},\overline{\mu}))(\overline{X}^{d_{0}}-\overline{v_{i}})(\overline{X}^{d_{0}}-\overline{\mu_{i}})^{\top}]\|_{op}. \label{Viikd_0}
\end{align}
\begin{lemma} \label{reduce_dim_lemma}
	For any $i,j,k\in[K]$ with $i\neq j$,
	\begin{align*}
	 V_{ij}^k	\le
	\max\left( \overline{V}_{ij}^k,\mathbb{E}_k[w_i(X,\mu)w_j(X,\mu)]\right) 
	,\\
    V_{i,j}^k 
	\le
	\max\left(    \overline{V}_{i,j}^k,\mathbb{E}_k[w_i(X,\mu)(1-w_i(X,\mu))]\right)
	\end{align*} 
\end{lemma}
The proof is similar to the one  in \cite{yan2017convergence}.
We include it  for our paper to be self contained. 

\begin{proof}
	We prove only the first inequality. The proof of the second inequality is similar. Note that $\left(X-v_i\right)\left(X-\mu_{j}\right)^{\top}$ is equal to
	\begin{equation}
	\Gamma^{\top}\left[\begin{array}{cc}
	\left(\overline{X}^{d_{0}}-\overline{v_i}\right)\left(\overline{X}^{d_{0}}-\overline{\mu_{j}}\right)^{\top} & \left(\overline{X}^{d_{0}}-\overline{v_i}\right)\left(\overline{X}^{\left[d-d_{0}\right]_{+}}\right)^\top
	\\
	\overline{X}^{\left[d-d_{0}\right]_{+}}\left(\overline{X}^{d_{0}}-\overline{\mu_{j}}\right)^{\top} & \overline{X}^{\left[d-d_{0}\right]_{+}}\left(\overline{X}^{\left[d-d_{0}\right]_{+}}\right)^\top
	\end{array}\right]\Gamma.  
	\nonumber
	\end{equation}
	Now, since $\Gamma$ is a rotation matrix and  the last $[d-d_0]_+$ coordinates of $\Gamma\mu_i$ are $0$ 
	we get that $w_i(X,\mu)=w_i(\overline{X}^{d_0},\overline{\mu})$.
	Therefore   $w_{i}(X,\mu),\left(\overline{X}^{d_{0}}-\overline{v_i}\right),\left(\overline{X}^{d_{0}}-\overline{\mu_{j}}\right)$ are independent of $\overline{X}^{\left[d-d_{0}\right]_{+}}$. Thus, 
	\begin{equation*}
	V_{ij}^k 
	 \le
	 \left\Vert \begin{array}{cc}
	\overline{V}^{ij}_k & 0\\
	0 & C^{ij}_k
	\end{array}\right\Vert _{op}\le\max\left(\overline{V}_{ij}^k(\mu,v),\|C^{ij}_k\|_{op}\right)
	\end{equation*}
	where $V^{ij}_k$ and $\overline{V}^{ij}_k$ are defined in  \eqref{Vijk}  and \eqref{Vijkd_0}, respectively and
	\begin{align*}
	C^{ij}_k&=\mathbb{E}_{k}\left[w_{i}\left(\overline{X}^{d_{0}},\overline{\mu}\right)w_{j}\left(\overline{X}^{d_{0}},\overline{\mu}\right)\overline{X}^{\left[d-d_{0}\right]_{+}}\left(\overline{X}^{\left[d-d_{0}\right]_{+}}{}\right)^{\top}\right]	\\
	&=
	\mathbb{E}_{\overline{X}^{d_{0}}}\left[w_{i}\left(\overline{X}^{d_{0}},\overline{\mu}\right)w_{j}\left(\overline{X}^{d_{0}},\overline{\mu}\right)\right]\mathbb{E}_{\overline{X}^{\left[d-d_{0}\right]_{+}}}\left[\overline{X}^{\left[d-d_{0}\right]_{+}}
	\left(\overline{X}^{\left[d-d_{0}\right]_{+}}\right)^\top\right]
	\\&=
	\mathbb{E}_{\overline{X}^{d_{0}}}\left[w_{i}\left(\overline{X}^{d_{0}},\overline{\mu}\right)w_{j}\left(\overline{X}^{d_{0}},\overline{\mu}\right)\right]I_{\left[d-d_{0}\right]_{+}}
	\end{align*}
	Since $w_i(X,\mu)=w_i(\overline{X}^{d_0},\overline{\mu})$, we may return to the original variables and write  
	\begin{equation*}
	\|C^{ij}_k\|_{op}=\mathbb{E}_{k}\left[w_{i}\left(X,\mu\right)w_{j}\left(X,\mu\right)\right].
	\end{equation*}
	Hence, the lemma follows.
\end{proof}

\begin{lemma} \label{Lemma:normbounds}
	fix $\lambda\in (0,\frac12)$. Let $(\mu_1^*, \ldots, \mu_K^*)$ be the centers of a $K$ component GMM. Let $X\sim \mathcal{N}(\mu_k^*, I_d)$ for some $k\in [K]$. Then, for any $i,j\in [K]$ and any $\mu, v \in \mathcal{U}_{\lambda}$,  
	\begin{equation} \label{Eq:norms_bound_final}
	\mathbb{E}_{k}[\|(X-v_i)(X -\mu_j)^\top\|_{op}^2] \le C_0\begin{cases}
	\max\left(d^{2},R_{i}^{4}\right) & i=j=k\\
	\max\left(d^{2},R_{ik}^{2}R_{jk}^{2}\right) & k\ne i,k\ne j\\
	\max\left(d^{2},R_{ij}^{4}\right) & k=i,k\ne j
	\end{cases} 
	\end{equation}
	where $C_0$ is a universal constant, for example we can take $C_0=14$.
\end{lemma}
\begin{proof}
	First, for any rank $1$ matrix $uv^{\top}$ it holds that $\|uv^T\|_{op}= \|u\|\cdot\|v\|$. Thus,
	\begin{equation} 
	\nonumber	\mathbb{E}_{k}\left[\|(X-v_i)(X -\mu_j)^{\top}\|_{op}^2\right]= \mathbb{E}_{k}\left[\|X-v_i\|^2\cdot\|X -\mu_j\|^2\right].
	\end{equation}
	Next, since $X\sim {\cal N}(\mu_k^*,I_d)$ we may write $X=\mu_k^*+\eta$, where
	$\eta\sim \mathcal N(0,I_d)$. Thus, 
	\begin{equation}
	\|X-v_i\|^2 \cdot \|X-\mu_j\|^2 = \|\mu_k^*-v_i + \eta\|^2 \cdot
	\|\mu_k^*-\mu_j + \eta\|^2.
	\nonumber    
	\end{equation}
	Let $\Gamma$ be a rotation matrix such that \begin{equation}
	\Gamma\left(\mu_{k}^{*}-v_i\right)=R_{ik}^{*}e_{1},\quad\Gamma\left(\mu_{k}^{*}-\mu_{j}\right)=R_{kj}^{*}\cos\left(\alpha\right)e_{1}+R_{kj}^{*}\sin\left(\alpha\right)e_{2}    \nonumber
	\end{equation}  
	where $R_{ik}^{*}= \|\mu_k^*-v_i\|, R_{kj}^{*}= \|\mu_k^*-\mu_j\|$ and $\alpha$ is the angle between $e_1$ and $\Gamma(\mu_k^*-\mu_j)$. Then by applying $\Gamma$ to  and using the rotation invariance of the Gaussian distribution,  
	\begin{equation}
	\|X-v_i\|^2  =  (R_{ik}^{*}+\eta_1)^2 + \eta_2^2 + \sum_{q>2} \eta_q^2 
	\nonumber
	\end{equation}
	and 
	\[
     \|X-\mu_j\|^2 =    (R_{kj}^*\cos\alpha+\eta_1)^2+ (R_{kj}^*\sin\alpha+\eta_2)^2
	+\sum_{q>2} \eta_q^2.
	\]
	It is easy to show that the expectation of the above expression is maximal when $\alpha=0$. In this case, we can write the expectation as follows \begin{equation}
	\E[\|X-v_i\|^2 \cdot \|X-\mu_j\|^2] \leq
	\mathbb{E}[ (A+C)\cdot (B+C)] \nonumber
	\end{equation} 
	where $A=(R_{ik}^{*}+\eta_1)^2$ follows a non-central $\chi^2$ distribution with
	one degree of freedom and non-centrality parameter $(R_{ik}^{*})^2$, 
	$C=\sum_{q\geq 2} \eta_q^2$ follows a central $\chi^2$ distribution with $d-1$ degrees of freedom, and $B= (R_{kj}^{*}+\eta_1)^2$. 
	Using known results on the moments of central and non-central $\chi^2$ random variables, 
	\begin{align*}
	\mathbb{E}[AB] &= \mathbb{E}[((R_{ik}^{*})^2 + 2R_{ik}^{*}\eta_1 + \eta_1^2) 
	((R_{kj}^{*})^2 + 2 R_{kj}^{*}\eta_1 + \eta_1^2) ] 
	\\ &=   (R_{ik}^{*} R_{kj}^{*})^2 + (R_{ik}^{*})^2 + (R_{kj}^{*})^2  
	+ 4 R_{ik}^{*}R_{kj}^{*} + 3.   
	\end{align*}
	
	and
	\begin{align*}
	\mathbb{E}[ (A+C)\cdot (B+C)] & =  \mathbb{E}[AB] + (\mathbb{E}[A]+
	\mathbb{E}[B]) \mathbb{E}[C] + \mathbb{E}[C^2] \nonumber \\
	&=    \mathbb{E}[AB] + [(R_{ik}^{*})^2 + (R_{kj}^{*})^2 + 2] (d-1) + d^2 - 1. 
	\end{align*}
	Since $\mu,v\in \mathcal{U}_{\lambda}$ it holds that $R_{ik}^*\le R_{ik} + \frac12R_i, R_{kj}^*\le  R_{kj} + \frac12 R_j$.
	
	Now we consider several different cases. First, for $i=j=k$ we have $R_{ik}=R_{jk}=0$ and $R_i=R_j$. Hence  
	\begin{equation}
	\mathbb{E}_{i}[\|(X-v_i)(X -\mu_i)^T\|_{op}^2] \le C_0\max(d^2,R_i^4).
	\nonumber
	\end{equation}
	Next, if $k$ is distinct from both $i$ and $j$, then $R_i\le R_{ik}$ and $R_j\le R_{kj}$. Hence, 
	
	\begin{equation}
	\mathbb{E}_{k}[\|(X-v_i)(X -\mu_j)^{\top}\|_{op}^2] \le C_0\max(d^2,R_{ik}^2R_{kj}^2).
	\nonumber
	\end{equation}
	Finally, we consider the case  where $j\ne i$  but $k$ is not distinct from both $i$ and $j$, without loss of generality $k=i$. Then $R_{ik}=0,R_{kj}=R_{ij}$. By definition, $R_j\le R_{ij}$ and $R_{i}\le R_{ij}$. Thus, \begin{equation}
	\mathbb{E}_{i}[\|(X-v_i)(X -\mu_j)^{\top}\|_{op}^2] \le C_0\max(d^2,R_{ij}^4).
	\nonumber
	\end{equation}
	
\end{proof}


We are now ready to prove the lemma. For clarity 
we present in two separate parts the proof
of Eq. (\ref{Vijbound}) and of Eq. (\ref{Viibound}).

\begin{proof}[Proof of  Eq. \eqref{Vijbound} in Lemma \ref{Lem:bound_expectation_grad}]
	The first step is to separate the expectation over the GMM to its $K$ components. By the triangle inequality,
	\begin{equation}
	V_{ij}(\mu,v)\leq \sum_{k} \pi_k V_{ij}^k(\mu,v)
	\label{eq:vij}
	\end{equation}
	with $V_{ij}^k$ as defined in \eqref{Vijk}. By Lemma \ref{reduce_dim_lemma}, for each $k$,
	\begin{equation}
	V_{ij}^k(\mu,v)
	\le\max(\overline{V}_{ij}^k(\mu,v),\mathbb{E}_{k}[w_{i}(X,\mu)w_{j}(X,\mu)])
	\label{Eq:reduce_dimij}
	\end{equation}
	with $\overline{V}_{ij}^k$ as defined in Eq. \eqref{Vijkd_0}. 
	We now separately analyze each of the two terms on the right hand size of \eqref{Eq:reduce_dimij}. 
	We start with the second term. When $k=i$, by Proposition \ref{Lem:w_exponentialy_small} 
	$$
	\mathbb{E}_{i}\left[w_{i}(X,\mu)w_{j}(X,\mu)\right]\le
	\mathbb{E}_{i}\left[w_{j}(X,\mu)\right] \leq
	\left(1+\theta\right)e^{-c\left(\lambda\right)R_{ij}^{2}}.
	$$
	By symmetry, the same bound holds also for $k=j$. 
	
	Next we bound $\overline{V}_{ij}^k$ and we shall later see that it is the largest of the two quantities in \eqref{Eq:reduce_dimij}. 
	Note that by the Cauchy-Schwarz inequality 
	\begin{equation*}
	\overline{V}_{ij}^k
	\le
	\sqrt{\mathbb{E}_{k}\left[\left(w_{i}(X,\mu)w_{j}(X,\mu)\right)^{2}\right]}\sqrt{\mathbb{E}_{k}\left[\left\|(\overline{X}^{d_{0}}-\overline{v_{i}})(\overline{X}^{d_{0}}-\overline{\mu_{j}})^{\top}\right\|_{op}^2\right]}.
	\end{equation*}
	By Lemma \ref{Lemma:normbounds} there exists a universal constant $C$ such that Eq \eqref{Eq:norms_bound_final} holds with dimension $d_0$.
	Thus, for the first term in Eq. (\ref{eq:vij}), with $k=i$, 
	\begin{align*}
	\overline{V}_{ij}^i(\mu,v)
	&\le
	\sqrt{\mathbb{E}_{i}[\left(w_{i}(X,\mu)w_{j}(X,\mu)\right)^{2}]}\sqrt{C}\max(d_{0},R_{ij}^2)
	\\&\le
	\sqrt{\mathbb{E}_{i}[w_{j}(X,\mu)]}\sqrt{C}\max(d_{0},R_{ij}^2)
	\end{align*}
	and by Proposition \ref{Lem:w_exponentialy_small}
	\begin{equation}
	\overline{V}_{ij}^i(\mu,v)
	\le 
	\sqrt{C\left(1+\theta\right)}e^{-\frac{c\left(\lambda\right)}{2}R_{ij}^{2}}\max\left(d_{0},R_{ij}^2\right). \label{eq:viji}
	\end{equation}
	
	Similarly, for $k=j$, \begin{equation}
		\overline{V}_{ij}^j(\mu,v)
	\le
	\sqrt{C\left(1+\theta\right)}e^{-\frac{c\left(\lambda\right)}{2}R_{ij}^{2}}\max\left(d_{0},R_{ij}^2\right).\label{eq:vijj}
	\end{equation} 
	Hence, in these two cases indeed $\overline{V}^{ij}_k$ is the dominant term.
	
	Finally we consider the case $k\ne i,j$. Again by 
	Proposition \ref{Lem:w_exponentialy_small} 
	\[
	\mathbb{E}_{k}\left[w_{i}(X,\mu)w_{j}(X,\mu)\right]\le
	\left(1+\theta\right)e^{-c\left(\lambda\right) \max(R_{ik},R_{jk})^2}.
	\]
	As for the first term, 
	\begin{align}
	\overline{V}_{ij}^k 
	&\le
	\sqrt{\mathbb{E}_{k}\left[\left(w_{i}(X,\mu)w_{j}(X,\mu)\right)^{2}\right]}\sqrt{C}\max\left(d_{0},R_{ik}R_{jk}\right) \nonumber\\
	&\leq \sqrt{C(1+\theta)} e^{-\frac{c(\lambda)}2\max(R_{ik},R_{jk})^2}
	\max(d_0,\max(R_{ik},R_{jk})^2). \label{eq:vijk}
	\end{align}
	Inserting \eqref{eq:viji}, \eqref{eq:vijj} and \eqref{eq:vijk} into \eqref{eq:vij} and summing over the components gives
	\begin{align*}
	V_{i,j}(\mu,v) \le & \sqrt{C}\sqrt{1+\theta}
	\Big[ (\pi_i+\pi_j)\max(d_0,R_{ij}^2)e^{-\frac{c(\lambda)}2 R_{ij}^2} 
	\\&+
	\sum_{k\neq i,j} \pi_k \max(d_0,\max(R_{ik},R_{jk})^2)
	e^{-\frac{c(\lambda)}2 \max(R_{ik},R_{jk})^2}
	\Big].
	\end{align*}
	Since the function $x^2e^{-tx^2}$ is monotonic decreasing for $x>\sqrt{1/t}$,
	and  $R_{\min}>\sqrt{2/c(\lambda)}$ we may replace 
	$R_{ij}$ by $\max(R_i,R_j)$ in the first term. Similarly, we may replace $R_{ik}$ by $R_i$ and $R_{jk}$ by $R_j$ in the second sum above. 
	This yields
	Eq. \eqref{Vijbound}. 
\end{proof}


\begin{proof}[Proof of  Eq. \eqref{Viibound} in Lemma \ref{Lem:bound_expectation_grad}]
	The first step is to separate the expectation over the GMM to its $K$ components. By the triangle inequality, 
	\begin{align}
	V_{ii}(\mu,v)&
	\le
	\sum_{k=1}^K\pi_{k}V_{ii}^k(\mu,v) 
	\label{eq:V_ii_sum}
	\end{align}
	with $V_{ii}^k$ as defined in \eqref{Viik}.
	We now bound each $V_{ii}^k$ separately. 
	First, by Lemma \ref{reduce_dim_lemma},   
	\begin{equation}
	V_{ii}^k 
	\le 
	\max(\overline{V}_{ii}^k, \mathbb{E}_k[(1-w_i(X,\mu))w_i(X,\mu)])  \label{Eq:reducedimXii}
	\end{equation}
	with $\overline{V}_{ii}^k$ as defined in \eqref{Viikd_0}. 
	We  now analyze each component separately. We first bound $\overline{V}_{ii}^k$ and we shall later see that it is the largest of the two quantities in \eqref{Eq:reducedimXii}. By the Cauchy-Schwarz inequality,
	\begin{equation*}
	\overline{V}_{ii}^k(\mu,v)^2
	\le
	\mathbb{E}_{k}\left[\left(w_{i}(X,\mu)(1-w_i(X,\mu))\right)^{2}\right]\cdot\mathbb{E}_{k}\left[\left\|(\overline{X}^{d_{0}}-\overline{v_{i}})(\overline{X}^{d_{0}}-\overline{\mu_{i}})^{\top}\right\|_{op}^2\right].
	\end{equation*}
	By Lemma 		\ref{Lemma:normbounds}, there exists  a universal constant $C$ such that Eq. \eqref{Eq:norms_bound_final} holds with dimension $d_0$.  
	Thus for $k=i$, 
	\begin{align*}
	\overline{V}_{ii}^i(\mu,v)
	&\le\sqrt{\mathbb{E}_{i}\left[\left(w_{i}(X,\mu)(1-w_i(X,\mu))\right)^{2}\right]}\sqrt{C}\max\left(d_{0},R_{i}^2\right)
	\\&\le\sqrt{\mathbb{E}_{i}\left[1-w_{i}(X,\mu)\right]}\sqrt{C}\max\left(d_{0},R_{i}^2\right)
	\end{align*}
	and by Corollary \ref{Lem:w_exponentialy_close_to_1}, 
	\begin{equation}
	\overline{V}_{ii}^i(\mu,v)
	\le
	\sqrt{C(K-1)(1+\theta)}e^{-\frac{c\left(\lambda\right)}{2}R_{i}^{2}}\max\left(d_{0},R_{i}^2\right).
	\label{eq:E_ii_wX}
	\end{equation}
	Similarly, we upper bound the second quantity on the right hand side of
	\eqref{Eq:reducedimXii} as follows, 
	$$
	\mathbb{E}_{i}\left[w_{i}(X,\mu)(1-w_i(X,\mu))\right] \leq  (K-1)(1+\theta)e^{-c\left(\lambda\right)R_{ij}^{2}}. 
	$$ 
	Thus $\overline{V}_{ij}^k(\mu,v)$ is the dominant term in the maximum in 
	\eqref{Eq:reducedimXii}. 

	Now, for $k\ne i$, \begin{align*}
	\overline{V}_{ij}^k(\mu,v)
	&\le
	\sqrt{\mathbb{E}_{k}\left[\left(w_{i}(X,\mu)(1-w_i(X,\mu))\right)^{2}\right]}\sqrt{C}\max\left(d_{0},R_{ik}^2\right)\\
	&\le\sqrt{1+\theta}e^{-\frac{c\left(\lambda\right)R_{ik}^{2}}{2}}\sqrt{C}\max\left(d_{0},R_{ik}^2\right)
	\end{align*}
	The function $x^2e^{-tx^2}$ is monotonic decreasing for $x>\sqrt{t^{-1}}$. Since $R_i\ge \sqrt{4c(\lambda)^{-1}}$, so does $R_{ik}$, and we may replace it
	in the equation above by $R_i$. Namely, 
	\begin{equation}
	\overline{V}_{ij}^k(\mu,v) \leq \sqrt{1+\theta} \sqrt{C}\max(d_0,R_i^2) e^{-c(\lambda)R_i^2/2}
	\label{eq:E_ki_wX}
	\end{equation}
	Inserting (\ref{eq:E_ii_wX}) and (\ref{eq:E_ki_wX}) into (\ref{eq:V_ii_sum}), and summing over all components yields Eq. \eqref{Viibound}. 
\end{proof}

\subsection{Completing the Proof of Theorem \ref{main_EM_population_theorem}}

The following lemma completes the proof of the Theorem.  
\begin{lemma}
	\label{lem:mu_plus_mu_i}
	Let $X\sim\GMM$  with $R_{\min}$ satisfying \eqref{minimal_separation}. Let $\mu^+$ be the population EM update \eqref{populationEmupdate}. Then for every $i\in [K]$ 
	it holds that $\|\mu_i^+-\mu_i^*\|\le \lambda R_i$.
\end{lemma}

\begin{proof}
	Our starting point is  Eq. \eqref{eq:wimumu*}, 	\begin{align*}
	\| \mathbb{E}_X[(w_i(X,\mu)- w_i(X,\mu^*))(X-\mu_i^*)] \|&
	\le
	\sum_{j=1}^K\sup_{\mu\in\U}V_{i,j}(\mu,\mu^{*})\|\mu_{j}-\mu_{j}^{*}\|.
	\end{align*} 
	We insert the bounds (\ref{Viibound}) and (\ref{Vijbound}) on $V_{i,i}$ and $V_{i,j}$, respectively, to the above.

	Since $\mu\in\U$ we may replace all $\|\mu_{k}-\mu_{k}^{*}\|$ in the expressions above  by $\lambda R_k$. Since $x^3 e^{-tx^2}$ is monotonic decreasing for all $x>\sqrt{3/2t}$,
    	we may replace all $R_i,R_j$ above by $R_{\min}$. Defining $U = \frac{16\left(K-1\right)\sqrt{C\left(1+\theta\right)}}{3\pi_{\min}}$, this gives
	\begin{align*}
	\| \mathbb{E}_X[(w_i(X,\mu)- w_i(X,\mu^*))(X-\mu_i^*)] \|&
	\le 
	\lambda R_{\min} \cdot \frac{ 3\pi_{\min}}{8}U \max\left(d_0,R_{\min}^2\right)e^{-\frac{c(\lambda)R_{\min}^2}{2}}.
	\end{align*}
	By Eqs. \eqref{eq:numeratorww} and  \eqref{Eq:wi_greater}, it follows that 
	\begin{equation}
	\|\mu_i^+-\mu_i^*\| \le \lambda R_{\min} \cdot e^{-\frac{c(\lambda)}2R_{\min}^2} \frac12 U \max(d_0,R_{\min}^2) \nonumber
	\end{equation}
	The separation condition \eqref{minimal_separation} suffices 
	to ensure that 
	$\|\mu_i^+-\mu_i^*\|\le \lambda R_{\min}$.

\end{proof}

\section{PROOFS FOR SECTION \ref{section:sample} }

\subsection{Preliminaries}

	


We recall basic definitions and  results on sub-Gaussian random variables. See e.g. \citep{vershynin2018high}. \begin{definition}
	\begin{enumerate}
		\item A random variable $X$ is called sub-Gaussian if there exists $t>0$ such that $\mathbb{E}\left[e^{\frac{X^2}{t^2}}\right]\leq 2$.
		 Its norm is  defined as 
$	\|X\|_{\psi_2} = \inf_{t>0}\mathbb{E}\left[e^{\frac{X^2}{t^2}}\right ] \leq 2$.
		\item A random vector $X\in \mathbb{R}^d$ is called sub-Gaussian if $\sup_{v\in S^{d-1}}\|X^{\top}v\|_{\psi_2}<\infty$. Its sub-Gaussian norm is defined as $\sup_{v\in S^{d-1}}\|X^{\top}v\|_{\psi_2}$. 
	\end{enumerate}
\end{definition}	
\begin{lemma} \label{norm_of_subGaussian}
	Let $X\in \mathbb{R}^d$ be a sub-Gaussian random vector with sub-Gaussian norm at most $R$. Let $X_1,\ldots, X_n$ be $n$ i.i.d. copies of $X$. Define $S_{n}=\frac1n\sum_{\ell=1}^nX_{\ell}$.
	Then, there exists a universal constant $c$ such that  for any $t>0$, 
	$\Pr\left(\left\|S_n- \mathbb{E}[X]\right\|>t \right) \le e^{-\frac{cnt^2}{R^2}+d\log3}$.
\end{lemma}

\begin{proof}
	Let $N_{\frac{1}{2}}$ be a $\frac{1}{2}$-net of $S^{d-1}$ and fix $v\in N_{\frac12}$. By definition $X^{\top}v$ is sub-Gaussian with $\|X^{\top}v\|_{\psi_2}\le R$. Write $X_{v,n}=\frac1n\sum_{\ell=1}^n(X_{\ell}-\mathbb{E}[X])^{\top}v$.  Then by Hoeffding's inequality there exists a universal constant $c$ such that,
	\begin{equation}
	\Pr\left(|X_{v,n}| >t\right)\le e^{- \frac{c nt^2}{\|X^{\top}v\|_{\psi_2}^2}} \le  e^{- \frac{c nt^2}{R^2}}.  \nonumber
	\end{equation}
	Next, we note that for any $x\in \mathbb{R}^d$ it holds that $\|x\|\le 2 \sup_{v\in N_{\frac{1}{2}}}v^{\top}x$. 
	As is well known, the size of an $\varepsilon$-net is bounded by 
	$|N_{\frac{1}{2}}|\le e^{d\log3}$ \citep[Corollary 4.2.13]{vershynin2018high}.  The lemma therefore follows from a union bound.
\end{proof}


The following lemma is key for proving uniform convergence. A version of this lemma appears in \cite{zhao2020statistical}.    
\begin{lemma}
	\label{uniform_convergence_lemma}Fix $0<\delta<1$. Let $B_{1},\ldots,B_{K}\subset\mathbb{R}^{d}$
	be Euclidean balls of radii $r_{1},\ldots,r_{K}\ge1$. Define $\mathcal{B}=\otimes_{k=1}^{K}B_{k} \subset \mathbb{R}^{Kd}$
	and $r=\max_{k\in[K]}r_{k}$. Let $X$ be a random vector in $\mathbb{R}^d$ and $W
	:\mathbb{R}^{d}\times{\cal B}\to \mathbb{R}^k$ where $k\leq d$. 
	Assume  the following  hold: 
	
	1. There exists a constant $L\ge 1$ such that for any $\mu\in\mathcal{B},\varepsilon>0$, and $\mu^{\varepsilon}\in \mathcal B$
	which satisfies $\max_{i\in[K]}\|\mu_{i}-\mu_{i}^{\varepsilon}\|\le\varepsilon$,
	then 
	$
	\mathbb{E}_X\left[\sup_{\mu \in\mathcal B}\|W(X,\mu)-W(X,\mu^{\varepsilon})\|\right]\le L\varepsilon
	$.

	2. There exists a constant $R$ such that for any $\mu\in {\cal{B}}$, $\|W(X,\mu)\|_{\psi_2}\le R$.
	
	Let $X_{1},\ldots,X_{n}$ be i.i.d.
	random vectors with the same distribution as $X$. Then there exists a universal constant $\tilde{c}$ such that with probability at least $1-\delta$,
	\begin{equation}\label{etaeq}
	\sup_{\mu\in\mathcal{B}}\left\Vert \frac{1}{n}\sum_{\ell=1}^{n}W\left(X_{\ell},\mu\right)-\mathbb{E}_X\left[W\left(X,\mu\right)\right]\right\Vert \le R\sqrt{\tilde{c}\frac{Kd\log\left(\frac{18nLr}{\delta}\right)}{n}}.
	\end{equation}
\end{lemma}

\begin{proof}
	For any $\varepsilon>0$, let $N_i$ be an $\varepsilon$-net of $B_i$ and define $N_\varepsilon=\otimes N_i$. Then, 
	\begin{align*}
	\left\Vert \frac{1}{n}\sum_{\ell=1}^{n}W\left(X_{\ell},\mu\right)-\mathbb{E}\left[W\left(X,\mu\right)\right]\right\Vert  
	& \le
	\left\Vert \mathbb{E}\left[W\left(X,\mu\right)\right]-\mathbb{E}\left[W\left(X,\mu^{\varepsilon}\right)\right]\right\Vert \\
	& +\left\Vert \frac1n\sum_{\ell=1}^{n}\left(W\left(X_{\ell},\mu\right)-W\left(X_{\ell},\mu^{\varepsilon}\right)\right) \right\Vert \\
	& +\left\Vert \frac{1}{n}\sum_{\ell=1}^{n}W\left(X_{\ell},\mu^{\varepsilon}\right)-\mathbb{E}\left[W\left(X,\mu^{\varepsilon}\right)\right]\right\Vert.
	\end{align*}
	Therefore for any $t>0$, 
	\[
	\Pr\left(\sup_{\mu\in\mathcal{B}}\left\Vert \frac{1}{n}\sum_{\ell=1}^{n}W\left(X_{\ell},\mu\right)-\mathbb{E}\left[W\left(X,\mu\right)\right]\right\Vert>t \right)\le\Pr\left(A\right)+\Pr\left(B\right)+\Pr\left(C\right)
	\]
	where the three events $A,B,C$ are given by
	\begin{align*}
	A=\left\{ \sup_{\mu\in\mathcal{B}}\left\Vert \sum_{\ell=1}^{n}\frac{1}{n}W\left(X_{\ell},\mu\right)-\sum_{\ell=1}^{n}\frac{1}{n}W\left(X_{\ell},\mu^{\varepsilon}\right)\right\Vert >\frac{t}{3}\right\} \\
	B=\left\{ \sup_{\mu\in{\cal B}}\left\Vert \mathbb{E}\left[W\left(X,\mu\right)\right]-\mathbb{E}\left[W\left(X,\mu^{\varepsilon}\right)\right]\right\Vert >\frac{t}{3}\right\} \\
	C=\left\{ \sup_{\mu^{\varepsilon}\in N_{\varepsilon}}\left\Vert \frac{1}{n}\sum_{\ell=1}^{n}W\left(X_{\ell},\mu^{\varepsilon}\right)-\mathbb{E}\left[W\left(X,\mu^{\varepsilon}\right)\right]\right\Vert >\frac{t}{3}\right\}.
	\end{align*}
	We first bound  $\Pr(A)$.	By Markov's inequality and the first condition of the lemma, 
	\begin{align*}
	\Pr\left(A\right)
	&\le
	\frac{3}{t}\mathbb{E}\left[\sup_{\mu\in\mathcal{B}}\left\Vert \sum_{\ell=1}^{n}\frac{1}{n}W\left(X_{\ell},\mu\right)-\sum_{\ell=1}^{n}\frac{1}{n}W\left(X_{\ell},\mu^{\varepsilon}\right)\right\Vert \right]
	\\&
	\le
	\frac{3}{t}\mathbb{E}\left[\sup_{\mu\in\mathcal{B}}\left\Vert W\left(X,\mu\right)-W\left(X,\mu^{\varepsilon}\right)\right\Vert \right]\le\frac{3\varepsilon L}{t}.
	\end{align*}
	Thus, for 
	$
	\frac{3\varepsilon L}{t}\le\frac{\delta}{2}
	$ 
	we have that $\Pr(A)<\frac{\delta}{2}$. Note also that for $t$ satisfying
	$\frac{3\varepsilon L}{t}\le\frac{\delta}{2}$,
	\[
	\sup_{\mu\in{\cal B}}\left\Vert \mathbb{E}\left[W\left(X,\mu\right)\right]-\mathbb{E}\left[W\left(X,\mu^{\varepsilon}\right)\right]\right\Vert 
	\le
	\mathbb{E}\left[\sup_{\mu\in{\cal B}}\left\Vert W\left(X,\mu\right)-W\left(X,\mu^{\varepsilon}\right)\right\Vert \right]
	\le\frac{t}{3}
	\]
	and hence $\Pr(B)=0$.
	
	Finally, we bound the probability of $C$. Here we use the second condition of the lemma, that $\|W(X,\mu)\|_{\psi_2} \le R$. 
	It follows from  Lemma \ref{norm_of_subGaussian}, that for any fixed $\mu^{\varepsilon}$,
	\[
	\Pr\left(\left\Vert \frac{1}{n}\sum_{\ell=1}^{n}W\left(X_{\ell},\mu^{\varepsilon}\right)-\mathbb{E}\left[W\left(X,\mu^{\varepsilon}\right)\right]\right\Vert >\frac{t}{3}\right)\le2e^{k\log3-\frac{cnt^{2}}{R^{2}}} \le2e^{d\log3-\frac{cnt^{2}}{R^{2}}}
	\]
	where $c$ is a universal constant. Since all the
	balls $B_{1},\ldots,B_{K}\subset\mathbb{R}^{d}$ are of radius at
	most $r$, it holds that $|N_{\varepsilon}|\le e^{\log(\frac{3r}{\varepsilon})Kd}$.
	Hence, taking a union bound,
	\[
	\Pr\left(\sup_{\mu^{\varepsilon}\in N_{\varepsilon}}\left\Vert \frac{1}{n}\sum_{\ell=1}^{n}W\left(X_{\ell},\mu^{\varepsilon}\right)-\mathbb{E}\left[W\left(X,\mu^{\varepsilon}\right)\right]\right\Vert >\frac{t}{3}\right)\le 2e^{Kd\log\frac{3r}{\varepsilon}+d\log 3-\frac{cnt^2}{R^2}}.
	\]
	The requirement that the right hand side of the above is smaller than $\frac{\delta}{2}$ implies 
	\[ 
	t\ge R \sqrt{\frac{Kd\log\frac{3r}{\varepsilon}+d\log 3+ \log\frac4{\delta}}{cn}}.
	\]
	Setting $\varepsilon= \frac{\delta}{6Ln}$, the condition $\Pr(A)\leq \frac{\delta}2$ implies $t>\frac1n$, which holds if $t$ satisfies the inequality above. 
	Hence, for $t>R\sqrt{\tilde{c}\frac{Kd\log\left(\frac{18nLr}{\delta}\right)}{n}}$ with a suitable universal constant $\tilde{c}$ it holds that $\Pr(A)+ \Pr(B)+ \Pr(C)\le \delta$.  
	\end{proof}

	\subsection{Proof of Lemma \ref{w_i_conc}}
	
	\begin{proof}
	The lemma will follow from Lemma \ref{uniform_convergence_lemma} by setting $X\sim{\cal}{N}(\mu_i^*, I_d)$, $\cal{B}= \U$ and $W=w_i(X,\mu)$. To this end we  show that the  conditions of Lemma \ref{uniform_convergence_lemma} hold: (i) There exists  $L>1$ such that for any $\varepsilon>0$,  $\mathbb{E}_i\left[\sup_{\mu}|w_i(X,\mu)- w_i(X,\mu^{\varepsilon})|\right]\le L\varepsilon $
	for all $\mu^{\varepsilon}\in \U$ with $\max_{i\in [K]}\|\mu_i-\mu_i^{\varepsilon}\|\le \varepsilon$.  
	(ii) The sub-Gaussian norm of  $w_i(X,\mu)$ for $X\sim \mathcal{N}(\mu_i^*,I_d)$ is bounded by a constant. The latter is clear as $w_i$ is bounded. For the former we use the mean value theorem. There exists a point $\tilde \mu$ such that
	\begin{equation}
	|w_i(X,\mu)-w_i(X,\mu^{\varepsilon})|
	=| \nabla w_i(X,\tilde{\mu})^\top (\mu - \mu^{\varepsilon} )|  \nonumber 
	\end{equation}
	Using the expressions \eqref{Eq:gradw_ii} and \eqref{Eq:gradw_ij} for the gradient of $w_i$ with respect to $\mu$,     
	\begin{align}
	|w_i(X,\mu)-w_i(X,\mu^{\varepsilon})|
	&\le  
	\sup_{ \tilde\mu\in \U} \|w_i(X,\tilde\mu)(1-w_i(X,\tilde\mu))(X-\tilde{\mu}_i) \| \varepsilon \nonumber
	\\&+ 
	\sum_{j\neq i}
	\sup_{\tilde\mu \in \mathcal{U}_{\lambda}}\|
	w_i(X,\tilde\mu)w_j(X,\tilde\mu)
	(X-\tilde{\mu}_j) \| \varepsilon
	\nonumber
	\end{align}
	Since $0\le w_i(X,\mu)\le 1$ we get \begin{equation}
	\mathbb{E}_i\left[\sup_{\mu \in \mathcal{U}_{\lambda}} |w_i(X,\mu)-w_i(X,\mu^\varepsilon)|\right] \le \varepsilon \sum_{j=1}^K 
	\mathbb{E}_i\left[\sup_{\mu \in \mathcal{U}_{\lambda}}\|X-\mu_j \|\right] . 
	\nonumber
	\end{equation} 
	Since $X\sim {\cal{N}}(\mu_i^*,I_d)$,  we may write $X= \eta+\mu_i^*$ where $X\sim {\cal{N}}(0,I_d)$. Therefore, \begin{equation*}
	    \mathbb{E}_i\left[\sup_{\mu \in \mathcal{U}_{\lambda}}\|X-\mu_j \|\right] \le \mathbb{E}\|\eta\|+ \sup_{\mu\in \mathcal{U}_{\lambda}}\|\mu_i^*-\mu_j\|  \le \sqrt{d} + 2R_{\max} .
	\end{equation*}
	It follows that
	\begin{equation}
	\mathbb{E}_i\left[\sup_{\mu \in \mathcal{U}_{\lambda}} |w_i(X,\mu)-w_i(X,\mu^\varepsilon)|\right]\le K(\sqrt{d}+2R_{\max})\varepsilon.  \nonumber
	\end{equation} 		
	The lemma follows by plugging $L=K(\sqrt{d}+2R_{\max})$ and $r= R_{\max}$ into \eqref{etaeq}.
	
\end{proof}

\subsection{Proof of Lemma \ref{sample_denom_lemma}}
\begin{proof}
	We denote  the set of all samples $X_{\ell}$ generated from the $i$-th component by $I_i$,  $n_i=|I_i|$ and $\tilde{\pi}_i = n_i/n$. 
	Since $w_i(X,\mu)\ge 0$ for any $X,\mu$ we can lower bound the sum in the event $D_i$ by considering only  the terms $w_i(X_{\ell},\mu)$ with $X_{\ell}\in I_i$,   
	\begin{equation*}
	\frac{1}{n}\sum_{\ell=1}^nw_i(X_{\ell},\mu) \ge \frac1n\sum_{X_{\ell}\in I_i}w_i(X_{\ell},\mu)= \tilde{\pi}_i\frac1{n_i}\sum_{X_{\ell}\in I_i}w_i(X_{\ell},\mu)  .
	\end{equation*}
	With a suitably large constant $C$, the 
	sample size requirement \eqref{eq:sample_requirement_denom} implies that $ n>300 \frac{\log\frac{8K}{\delta}}{\pi_{\min}} $. 
	By the multiplicative form of the Chernoff bound for Bernoulli random variables, see e.g.  \citep[Exercise 2.3.5]{vershynin2018high}, we have for any $\delta\in (0,1)$ that $|\tilde{\pi}_i-\pi_i|\le \frac1{10}\pi_i$. Therefore, $\tilde{\pi}_i\ge \frac{9}{10}\pi_i$ and thus \begin{equation*}
	\frac{1}{n}\sum_{\ell=1}^nw_i(X_{\ell},\mu) \ge \frac{9\pi_i}{10}\frac1{n_i}\sum_{X_{\ell}\in I_i}w_i(X_{\ell},\mu) .
	\end{equation*} 
	Now, defining
	$ 
	d_i= \sup_{\mu \in \mathcal{U}_{\lambda}}\left(\frac1{n_i}\sum_{X_{\ell}\in I_i}w_i(X_{\ell},\mu)- \mathbb{E}_i\left[w_i(X,\mu)\right]\right)
	$ 
	we have
	\begin{equation*}
	\inf_{\mu\in\mathcal{U}_{\lambda}}\frac{1}{n_i}
	\sum_{X_\ell\in I_i} w_i(X_{\ell},\mu) 
	\ge
	\inf_{\mu\in\mathcal{U}_{\lambda}}\mathbb{E}_i\left[w_i(X,\mu)\right] - d_i.
	\end{equation*}
	Note that by Lemma \ref{w_i_conc}, with probability at least $1-\frac{\delta}{4K}$,  Eq. \eqref{eq:w_i_conc} holds.
	Since $n_i\ge \frac9{10}n\pi_i$, we may replace $n_i$ in Eq. \eqref{eq:w_i_conc} by $n\pi_i$, and increase the relevant constants to $\tilde{c}_1=\frac{10}9\tilde{c}$
	and $\tilde{C_1}= \frac{10}9 C$. 
	It thus follows that 
	\begin{align*}
	\inf_{\mu\in \mathcal{U}_{\lambda}} \frac{1}{n}\sum_{\ell=1}^nw_i(X,\mu)  &\ge\frac{9\pi_{i}}{10}\left(\inf_{\mu\in\mathcal{U}_{\lambda}}\mathbb{E}_{i}[w_{i}(X,\mu)]-\sqrt{\tilde{c}_1\frac{Kd\log(\frac{   \tilde{C_1}n}{\delta})}{n\pi_i}} \right).
	\end{align*}
	The condition on the sample size \eqref{eq:sample_requirement_denom} implies that $\sqrt{\tilde{c}_1\frac{Kd\log(\frac{   \tilde{C_1}n}{\delta})}{n\pi_i}} \le \frac1{10}$ and therefore, 
	\begin{equation*}
	\inf_{\mu\in \mathcal{U}_{\lambda}}  \frac{1}{n}\sum_{\ell=1}^nw_i(X,\mu) \ge\frac{9\pi_{i}}{10}\left(\inf_{\mu\in \mathcal{U}_{\lambda}} \mathbb{E}_{i}[w_{i}(X,\mu)]-\frac{1}{10}\right).
	\end{equation*}By Corollary \ref{Lem:w_exponentialy_close_to_1},
	\begin{equation*}
	\inf_{\mu\in \mathcal{U}_{\lambda}}\frac{1}{n}\sum_{\ell=1}^nw_i(X,\mu)\ge \frac{9\pi_i}{10}\left(\frac{9}{10}-(K-1)(1+\theta)e^{-c(\lambda)R_{\min}^2} \right). 
	\end{equation*}
	Under the separation requirement 
	\eqref{R_mindenombig}, it holds that $(K-1)(1+\theta)e^{-c(\lambda)R_{\min}^2} \le \frac1{15}$.   Hence, 
    the event $D_i$ \eqref{sample_denom_lemma_event} occurs with probability at least $1-\frac{\delta}{2K}$. 
\end{proof}

\subsection{Proof of Lemma \ref{Lemm:w_i(x-mu)sub-Gaussnorm}}

\begin{proof}
	Write $X=\eta+Z$ where $\eta\sim{\cal N}\left(0,I_{d}\right)$
	and $Z\in\{\mu_1^*,\ldots,\mu_K^*\}$ has a distribution  $\Pr\left(Z=\mu_{j}^{*}\right)=\pi_{j}$. 
	First we prove \eqref{eq:w_i(x-mu^*)sub-Gaussnorm}. 
	By the triangle inequality
	\[
	\|w_{i}\left(X,\mu\right)\left(X-\mu_{i}^{*}\right)\|_{\psi_{2}}\le\|w_{i}\left(X,\mu\right)\eta\|_{\psi_{2}}+\|w_{i}\left(X,\mu\right)\left(Z-\mu_{i}^{*}\right)\|_{\psi_{2}}.
	\]
	Since $w_{i}\le1$ and $\eta\sim \mathcal{N}(0,I_d)$, it follows that $\|w_{i}\left(X,\mu\right)\eta\|_{\psi_{2}}\le\|\eta\|_{\psi_{2}}$  \citep[Lemma B.1 part 5]{zhao2020statistical}. 
	Using the explicit formula for the moment generating function of a chi-squared distribution with $1$ degree of freedom, $\mathbb{E}[
	\exp(1/t^2(\eta^{\top}s)^2)]=(1-2/t^2)^{-1/2}$. It follows that $\|\eta\|_{\psi_2} \le 8$ and hence $\|w_i(X,\mu)\eta\|_{\psi_2} \le 8$.
	Next, we analyze the sub-Gaussian norm of the second term.
	We show that $\|w_i(X,\mu)(Z-\mu_i^*)\|_{\psi_2}\le 8\frac{1+2\lambda}{1-2\lambda}$. To this end we show that for  $t=8\frac{1+2\lambda}{1-2\lambda}$ and any $s\in S^{d-1}$, 
	\begin{equation}
	\sum_{j=1}^{K}\pi_{j}\mathbb{E}_{j}\left[\exp\left(\frac1{t^2}\left(w_{i}\left(X,\mu\right)\left(\mu_{j}^{*}-\mu_{i}^{*}\right)^{\top}s\right)^{2}\right)\right]\le 2.
	\label{eq:sub_gauss_componsnent}
	\end{equation}
	First, for $j=i$, $Z-\mu_i^*=0$ and thus the expectation is $1$.
	Now, consider any $j\ne i$. It holds that $(\mu_j^*-\mu_i^*)^{\top}s\le R_{ij}$. Therefore, 
	\[
	\mathbb{E}_{j}\left[\exp\left(\frac1{t^2}\left(w_{i}\left(X,\mu\right)\left(\mu_{j}^{*}-\mu_{i}^{*}\right)^{\top}s\right)^{2}\right)\right]
	\le
	\mathbb{E}_{j}\left[\exp\left(\frac1{t^2}\left(R_{ij}w_{i}\left(X,\mu\right)\right)^{2}\right)\right].
	\]
	By Equations \eqref{wi1d} and \eqref{wi1d2},   $w_{i}\left(\eta+\mu_{j}^{*},\mu\right)\le\frac{1}{1+\frac{\pi_{j}}{\pi_{i}}e^{A\nu+B}}=\tilde{w}_i(A,B,\nu)$
	where $\nu\sim{\cal N}\left(0,1\right)$, $A=\|\mu_{i}-\mu_{j}\|$
	and $B=\frac{1}{2}\|\mu_{j}^{*}-\mu_{i}\|^{2}-\frac{1}{2}\|\mu_{j}-\mu_{j}^{*}\|^{2}.$ This allows bounding the expectation over the $d$-dimensional random vector $\eta$ by an expectation over a univariate random variable $\nu$. 
	\begin{equation}
	\mathbb{E}_{j}\left[\exp\left(\frac1{t^2}\left(w_{i}\left(X,\mu\right)\left(\mu_{j}^{*}-\mu_{i}^{*}\right)^{\top}s\right)^{2}\right)\right]
	\le
	\mathbb{E}_{\nu}\left[\exp\left(\frac1{t^2}R_{ij}^{2}\tilde{w}_i(A,B,\nu)^2\right)\right] = E.
	\nonumber
	\end{equation}
	Next, we split the expectation over $\nu$ to two cases as follows,
	\begin{align}
	E  &= 
	\mathbb{E}\left[\exp\left(\frac1{t^2}R_{ij}^{2}\tilde{w}_i(A,B,\nu)^{2}\right)\Big|\nu<\frac{-B}{2A}\right]\Pr\left(\nu<\frac{-B}{2A}\right)
	\nonumber
	\\&
	+
	\mathbb{E}\left[\exp\left(\frac1{t^2}R_{ij}^{2}\tilde{w}_i(A,B,\nu)^{2}\right)\Big|\nu>\frac{-B}{2A}\right]\Pr\left(\nu>\frac{-B}{2A}\right)
	= E_1 + E_2.
	\label{eq:E_bound}
	\end{align}
	We now show  that $E_1\le \frac12$ and $E_2\le \frac32$, from which it follows that $E\le2$.
	
	First, consider the term $E_1$ in (\ref{eq:E_bound}). Note that $\Pr\left(\nu<\frac{-B}{2A}\right)\le e^{-\frac{B^{2}}{8A^{2}}}.$
	By Lemma \ref{Lemm:ABbounds},  $A\le (1+2\lambda)R_{ij}$ and  $B \ge \frac12 (1-2\lambda)R_{ij}^2$.
	It follows that 
	$
	\Pr\left(\nu<\frac{-B}{2A}\right)\le e^{-\frac{1}{32}\left(\frac{1-2\lambda}{1+2\lambda}\right)^{2}R_{ij}^{2}}
	$.
	Since $\tilde{w}_i(A,B,\nu) \leq 1$, we thus obtain by inserting $t=8\frac{1+2\lambda}{1-2\lambda}$, 
	\begin{equation}
	E_1 \leq 
	e^{\frac{R_{ij}^{2}}{t^{2}}} e^{-\frac{1}{32}\left(\frac{1-2\lambda}{1+2\lambda}\right)^{2}R_{ij}^{2}} =
	e^{-\frac{1}{64}\left(\frac{1-2\lambda}{1+2\lambda}\right)^{2}R_{ij}^{2}}.\nonumber
	\end{equation}
	Therefore,  for $R_{\min}$ satisfying \eqref{eq:minimal_sep_gauss_norm}, 
	$
	E_1\le \frac12.
	$

	Second, consider the term $E_2$ in (\ref{eq:E_bound}). Since $\nu>\frac{-B}{2A}$, then $A\nu+B>\frac{B}{2}$. 
	Thus,  $1/\tilde{w}_i(A,B,\nu)^2=(1+\frac{\pi_j}{\pi_i}e^{A\nu+B})^2>\frac{\pi_j^2}{\pi_i^2}e^B$. 
	That is, $\tilde{w}_i(A,B,\nu)^2\le \frac{\pi_i^2}{\pi_j^2}e^{-B}$.  By 	 Lemma \ref{Lemm:ABbounds},  $B\ge\frac{\left(1-2\lambda\right)}{2}R_{ij}^{2}$. 	
	Hence, $\tilde{w}_i(A,B,\nu)^2\le \frac{\pi_{i}^{2}}{\pi_{j}^{2}}e^{-\frac{\left(1-2\lambda\right)}{2}R_{ij}^{2}} $.
	Note that $\Pr\left(\nu>\frac{-B}{2A}\right)\le1$, hence it follows by plugging in $t=8\frac{1+2\lambda}{1-2\lambda}$,
	\begin{equation}
	E_2 
	\le
	\exp\left(\frac1{64(1+2\lambda)^2}(1-2\lambda)^2R_{ij}^{2}\frac{\pi_{i}^{2}}{\pi_{j}^{2}}e^{-\frac{\left(1-2\lambda\right)}{2}R_{ij}^{2}}\right).
	\nonumber
	\end{equation}
	The condition that the right hand side of the above is smaller than $\frac32$ 
	can be written as $we^{-w}\le a$, where $w= \frac{1-2\lambda}{2}R_{ij}^{2}$ and $a= \frac{\pi_{j}^{2}}{\pi_{i}^{2}}32\log\frac{3}{2}\frac{\left(1+2\lambda\right)}{1-2\lambda}^{2}$. Since for $w\ge 2\log\frac1a$, it holds that $we^{-w}<a$, we get    for our case that for $R_{\min}$ satisfying \eqref{eq:minimal_sep_gauss_norm},
	$
	E_2\le \frac32
	$.

	Since $E\le 2$ 
	, Eq. \eqref{eq:sub_gauss_componsnent}  holds. 
	Therefore $\|w_i(X,\mu)(Z-\mu_i^*)\|_{\psi_2}\le 8\frac{1+2\lambda}{1-2\lambda}$. Since $\|w_i(X,\mu)\eta\|_{\psi_2}\le 8$,  we get Eq. \eqref{eq:w_i(x-mu^*)sub-Gaussnorm}.
	
	The proof of Eq. \eqref{eq:w_i(x-mu)sub-Gaussnorm} is similar.
	We  analyze the sub-Gaussian norm of  $\|w_{i}\left(X,\mu\right)(Z-\mu_i)\|_{\psi_{2}}$.  
	Similarly to Eq. \eqref{eq:sub_gauss_componsnent}, we decompose the expectation to components. First consider the $i$'th component.
	Since $\|\mu_{i}^{*}-\mu_{i}\|\le\lambda R_{i}$, we have
	for all $s \in S^{d-1}$ that $(\mu_i^*-\mu_i)^{\top}s\le \lambda R_i$. Thus, 
	\[
	\mathbb{E}_{i}\left[\exp\left(\frac1{t^2}\left(w_{i}\left(X,\mu\right)\left(\mu_{i}^{*}-\mu_{i}\right)^{\top}s\right)^{2}\right)\right]
	\le
	\mathbb{E}_{i}\left[\exp\left(\frac1{t^2}\left(w_{i}\left(X,\mu\right)\lambda R_{i}\right)^{2}\right)\right].
	\]
	Hence for $t\ge\frac{\lambda R_{i}}{\sqrt{\log2}}$,  the last expression is smaller than $2$.
	Next, for any component $j$ with $j\ne i$ and any $s\in S^{d-1}$, we have $(\mu_j^*-\mu_i)^{\top}s\le R_{ij}+\|\mu_i-\mu_i^*\| \le \frac{3}{2}R_{ij}$. Hence,
	\begin{align*}
	\mathbb{E}_{j}\left[\exp\left(\frac1{t^2}\left(w_{i}\left(X,\mu\right)\left(\mu_{j}^{*}-\mu_{i}\right)^{\top}s\right)^{2}\right)\right]
	& \le
	\mathbb{E}_{j}\left[\exp\left(\frac1{t^2}\left(w_{i}\left(X,\mu\right)\frac{3}{2}R_{ij}\right)^{2}\right)\right].
	\end{align*}
	Since for $t\ge 8\frac{1+2\lambda}{1-2\lambda}$, $\mathbb{E}_{\eta\sim{\cal N}\left(0,I_{d}\right)}\left[\exp\left(\frac1{t^2}\left(w_{i}\left(\eta+\mu_{j}^{*},\mu\right)R_{ij}\right)^{2}\right)\right]\le2$, it follows that for $t\ge 12\frac{1+2\lambda}{1-2\lambda}$, the above is smaller than $2$.
	Thus, for any $s\in S^{d-1},$
	$t\ge\max\left(12\frac{1+2\lambda}{1-2\lambda},\frac{\lambda R_{i}}{\sqrt{\log2}}\right)=t_0$, 
	\[
	\mathbb{E}_{X}\left[\exp\left(\frac1{t^2}\left(w_{i}\left(X,\mu\right)\left(X-\mu_{i}\right)^{\top}s\right)^{2}\right)\right]\le2.
	\]
	Thus, $\|w_i(X,\mu)(Z-\mu_i)\|_{\psi_2}\le t_0$. Since $\|w_i(X,\mu)\eta\|_{\psi_2}\le 8$, Eq. \eqref{eq:w_i(x-mu)sub-Gaussnorm} follows.

\end{proof}

\subsection{Proof of Lemma \ref{lem:concentration_numerator}}

We first present the following auxiliary lemma. 

\begin{lemma}\label{l}
	Let $X\sim \GMM$ with $R_{\min}$ satisfying \eqref{eq:minimal_sep_gauss_norm}. Fix $\lambda\in (0,\frac12)$. For each $\mu\in {\cal{U}}_{\lambda}$ let $\mu^{\varepsilon}\in\U$ be such that $\max_{i\in[K]}\|\mu_i-\mu_i^{\varepsilon}\|<\varepsilon$. 
	Then, 	for $v\in\{\mu,\mu^*\}$, 
	\begin{equation}
	\label{eq:b4equation}
	\mathbb{E}_X\left[\sup_{\mu\in {\cal{U}}_{\lambda}}\|(w_i(X,\mu)-w_i(X,\mu^{\varepsilon}))(X-v_i)\|\right] \le K(\sqrt{d}+2R_{\max})^2\varepsilon.
	\end{equation}
\end{lemma}

\begin{proof}
	By the mean value theorem and  the expression for $\nabla_\mu w_i(X,\mu)$,   \eqref{Eq:gradw_ii} and \eqref{Eq:gradw_ij}, 
	\begin{align}
	\mathbb{E}_X\left[\sup_{\mu\in \mathcal{U}_{\lambda} } \|(w_i(X,\mu)-w_i(X,\mu^{\varepsilon}))(X-v_i)\|\right] 
	& \le   \sum_{j=1}^K\sup_{\mu\in\U}V_{ij}(\mu,v)  \varepsilon.
	\nonumber
	\end{align} 
	Since $0\le w_i(X,\mu)\le 1$, we get, 
	\begin{equation}
	\mathbb{E}_X\left[\sup_{\mu\in \mathcal{U}_{\lambda} } \|(w_i(X,\mu)-w_i(X,\mu^{\varepsilon}))(X-v_i)\|\right] 
	\le
	\sum_{j=1}^K \mathbb{E}_X\left[\sup_{\mu \in \mathcal{U}_{\lambda}}\|X-\mu_j \|\|X-v_i\|\right] \varepsilon. 
	\nonumber
	\end{equation} 
	Now, 
	$$
	\mathbb{E}_{X}\left[\sup_{\mu\in\mathcal{U}_{\lambda}}\|X-\mu_{j}\|\right]=\sum_{k=1}^{K}\pi_{k}\mathbb{E}_{k}\left[\sup_{\mu\in\mathcal{U}_{\lambda}}\|X-\mu_{j}\|\right]\le\sqrt{d}+2R_{\max}.
	$$
	Similarly, 
	$\mathbb{E}_X\left[\sup_{\mu \in \mathcal{U}_{\lambda}}\|X-v_i \| \right] \le  \sqrt{d} + 2R_{\max}$.  Eq. \eqref{eq:b4equation} now follows.
\end{proof}

\begin{proof}[Proof of Lemma \ref{lem:concentration_numerator}]
	The lemma will follow from Lemma \ref{uniform_convergence_lemma} with  $X\sim \GMM$, $\cal{B}= \U$, 
	$W=w_i(X,\mu)(X-\mu_i^*)$ and probability $\delta_0= \frac{\delta}{2K}$. To this end we need to show that the two conditions of Lemma \ref{uniform_convergence_lemma} hold: (i)  For any $\varepsilon>0$,  Eq. \eqref{eq:b4equation} holds 
	for all $\mu^{\varepsilon}\in \U$ with $\max_{i\in [K]}\|\mu_i-\mu_i^{\varepsilon}\|\le \varepsilon$. (ii) The sub-Gaussian norm of  $w_i(X,\mu)(X-\mu_i^*)$ for $X\sim \GMM$ is bounded by  $\frac{16}{1-2\lambda}$. The former follows from Lemma \ref{l}.  The latter follows from Lemma    \ref{Lemm:w_i(x-mu)sub-Gaussnorm} for $R_{\min}$ satisfying \eqref{eq:minimal_sep_gauss_norm}.
\end{proof}

\section{PROOFS FOR THE GRADIENT EM ALGORITHM}\label{app:gradient}

\begin{proof}[Proof of Theorem \ref{grad_EM_population_theorem}]
	Consider the error of the estimate for the $i$-th center after a single gradient EM update \eqref{gradient_EM_update}. By the triangle inequality,
	\begin{align}
	\|\mu_{i}^{+}-\mu_{i}^{*}\|&\le\|\mu_{i}-\mu_{i}^{*}
	+s\mathbb{E}_X[w_i(X,\mu^*)(X-\mu_i)]\|\nonumber\\&+s\|\mathbb{E}_X[(w_i(X,\mu)-w_i(X,\mu^*))(X-\mu_i)]\|. 
	\nonumber
	\end{align}
	We now separately upper bound each of the two terms above. For the first term, recall that $\mathbb{E}_X[w_i(X,\mu^*)] = \pi_i$ and by Lemma \ref{lem:popEMmustarfixed},  $\mathbb{E}_X[w_i(X,\mu^*)(X-\mu_i)]= \pi_i(\mu_i^*-\mu_i) $. 
	Hence, for any step size $s<1/\pi_i$,  
	$$
	\|\mu_{i}-\mu_{i}^{*}+s\mathbb{E}_X[w_i(X,\mu^*)(X-\mu_i)]\| \le (1-s\pi_i)\|\mu_i-\mu_i^*\|.
	$$
	Next, to bound the second term we use the expressions  in Eqs. \eqref{Eq:gradw_ii} and \eqref{Eq:gradw_ij}, 
	\begin{align}
	\|\mathbb{E}_X[(w_i(X,\mu)-w_i(X,\mu^*))(X-\mu_i)]\| 
	& \le 
	\sum_{j=1}^K\sup_{\mu \in \U}V_{i,j}(\mu,\mu)\|\mu_j-\mu_j^*\| 
	\label{eq:grad_bound_w}
	\end{align} 
	with  $V_{i,j}$ and $V_{i,i}$ as defined in \eqref{V_ij} and \eqref{V_ii}, respectively. 
	The proof proceeds similarly to that of the original EM algorithm. First, using the bounds \eqref{Viibound} and \eqref{Vijbound} in Lemma \ref{Lem:bound_expectation_grad}, for $R_{\min}$ satisfying the separation condition \eqref{minimal_separation}, it holds that $$\|\mathbb{E}_X\left[\left(w_i(X,\mu)-w_i(X,\mu^*)\right)(X-\mu_i)\right]\|\le \frac38\pi_{\min}.$$
	Therefore, 
	\begin{equation}
	\|\mu_i^+-\mu_i^*\| \le (1-\frac58s\pi_i) E\left(\mu\right).
	\nonumber
	\end{equation}
	We finish by showing that $\mu^+\in \U$. Replacing $\|\mu_j-\mu_j^*\|$ by $\lambda R_j$ in \eqref{eq:grad_bound_w}, and replacing $R_k$ by $R_{\min}$ in the bounds in  \eqref{Viibound} and \eqref{Vijbound} we get
	\begin{align*}
	\| \mathbb{E}_X[(w_i(X,\mu)- w_i(X,\mu^*))(X-\mu_i)] \|&
	\le
		\lambda R_{\min} \cdot \frac{ 3\pi_{\min}}{8} U\max\left(d_0,R_{\min}^2\right)e^{-\frac{c(\lambda)R_{\min}^2}{2}}.
	\end{align*}
	with $U= \frac{16\left(K-1\right)\sqrt{C\left(1+\theta\right)}}{3\pi_{\min}}$.
	Under the separation condition \eqref{minimal_separation}, the right hand side of the above is upper bounded by $\frac38\pi_{\min}$. Therefore,
	\begin{align*}
	\|\mu_{i}^{+}-\mu_{i}^{*}\|&
	\le
	(1-s\pi_{i}+\frac38s\pi_i)\lambda R_{i} <\lambda R_{i}.
	\end{align*}
\end{proof}

The next lemma presents a concentration result for the sample EM update.

\begin{lemma} \label{sample_numerator_lemma_and_Grad}
	Fix $\delta\in (0,1),\lambda\in (0,\frac12)$. Let $X_1, \ldots, X_n\sim\GMM$ with 	 $R_{\min}$ satisfying \eqref{eq:minimal_sep_gauss_norm}. 
	For $i\in [K]$
	define  $S_i^g=\frac1n\sum_{\ell=1}^nw_i(X_{\ell},\mu)(X_{\ell}-\mu_i)$ and the event
	\begin{equation}
	N_i^g =\left\{\sup_{\mu\in \mathcal{U}_{\lambda}} \left\|S_i- \mathbb{E}[w_i(X,\mu)(X-\mu_i)] \right\| \le  C \max\left(\frac{1}{1-2\lambda}, \lambda R_i\right)\sqrt{\frac{Kd\log\frac{\tilde{C}n}{\delta}}{n}}\right\}  \label{grad_evet}
	\end{equation}
	where $C$ is a suitable universal constant and $\tilde{C}= 18K^2R_{\max}(\sqrt d+2R_{\max})^2$. Then $N_i$ occurs with probability at least $1-\frac{\delta}{K}$.
\end{lemma}
\begin{proof}
	The lemma will follow from Lemma \ref{uniform_convergence_lemma} by setting $X\sim \GMM$, $\cal{B}= \U$ and $W=w_i(X,\mu)(x-\mu_i)$. To this end we need to show that the two conditions of Lemma \ref{uniform_convergence_lemma} hold: (i)  For any $\varepsilon>0$, Eq \eqref{eq:b4equation} holds
	for all $\mu^{\varepsilon}$ with $\max_{i\in [K]}\|\mu_i-\mu_i^{\varepsilon}\|\le \varepsilon$. (ii) The sub-Gaussian norm of  $w_i(X,\mu)(X-\mu_i)$ for $X\sim \GMM$ is bounded by  $C\max(\frac{1}{1-2\lambda}, \lambda R_i)$. The former follows from Lemma \ref{l}.  The latter follows from Lemma    \ref{Lemm:w_i(x-mu)sub-Gaussnorm} for $R_{\min}$ satisfying \eqref{eq:minimal_sep_gauss_norm}.
\end{proof}

With the pieces in place we now prove Theorem \ref{sample_grad_em_main_theorem}.
\begin{proof}[Proof of Theorem \ref{sample_grad_em_main_theorem}]
	
	Consider  the error of the $i$-th cluster of the sample gradient EM update \eqref{eq:sample_grad_em_update},
	\begin{align}
	\|\mu_{i}^{*}-\mu_{i}^{+}\|
	&\le\left\Vert \mu_{i}-\mu_{i}^{*}-s\mathbb{E}[w_{i}(X,\mu)(X-\mu_{i})]\right\Vert
	\nonumber\\ &
	+s\left\Vert \mathbb{E}[w_{i}(X,\mu)(X-\mu_{i})]-\frac{1}{n}\sum_{\ell=1}^{n}w_{i}(X_{\ell},\mu)(X_{\ell}-\mu_{i})\right\Vert \nonumber 
	\end{align}
	Theorem \ref{grad_EM_population_theorem} implies that for $R_{\min}$ satisfying \eqref{minimal_separation}, it holds that
	$$
	\left\Vert \mu_{i}-\mu_{i}^{*}-s\mathbb{E}[w_{i}(X,\mu)(X-\mu_{i})]\right\Vert \le \gamma \min( E(\mu), \lambda R_i)
	$$
	with  $\gamma =1-\frac58s\pi_{\min}$.
	We therefore bound the second term above. Since the requirement on $R_{\min}$ \eqref{minimal_separation} is more restrictive than the requirement \eqref{eq:minimal_sep_gauss_norm} we may  invoke Lemma \ref{sample_numerator_lemma_and_Grad} and obtain that with probability at least $1- \frac{\delta}K$, the event $N_i^g$ \eqref{grad_evet}  occurs. Thus,
	\begin{align}
	\|\mu_{i}^{*}-\mu_{i}^{+}\|& \le  \gamma \min( E(\mu), \lambda R_i)+sC\max\left(\frac{1}{1-2\lambda},\lambda R_{i}\right)\sqrt{\frac{Kd\log\left(\frac{\tilde{C}n}{\delta}\right)}{n}}.
	\label{bound_gradientEM}
	\end{align}
	The sample size condition \eqref{minimal_sample_grad},    implies that $Cs\max\left(\frac{1}{1-2\lambda},\lambda R_{i}\right)\sqrt{\frac{Kd\log\left(\frac{\tilde{C}n}{\delta}\right)}{n}}\le  \frac38s\pi_{\min}\lambda R_{i} \le  \lambda(1-\gamma) R_i  $. Taking a union bound over the $K$ components, $\mu^+\in \U$ with probability at least $1-\delta$.  We can therefore iteratively apply \eqref{bound_gradientEM} and obtain  \begin{align*}
	\|\mu_i^t-\mu_i^*\| &\le \gamma^t E(\mu)+sC\frac{1}{1-\gamma} \max\left(\frac1{1-2\lambda}, \lambda R_i\right) \sqrt{\frac{Kd\log\left(\frac{\tilde{C}n}{\delta}\right)}{n}}.
	\end{align*}
	Since $\frac{s}{1-\gamma} = \frac8{5\pi_i}$, we get Eq. \eqref{convergence_up_to_statistical_errorgrad}.
\end{proof}
\section*{Acknowledgements}
We thank the associate editor and anonymous referee for several constructive suggestions.


\bibliographystyle{imsart-number} 
\bibliography{main.bib}       


\end{document}